\documentclass[preprint,12pt,3p]{elsarticle}
\usepackage{natbib}
\setcitestyle{authoryear,open={((},close={))}} 
\usepackage{listings}
\usepackage{amsmath}
\usepackage{xspace}
\usepackage{hyphenat}
\usepackage{graphicx}
\usepackage{rotating}
\usepackage{amsmath}
\usepackage{multicol}
\usepackage{array}
\usepackage{stfloats}
\usepackage{blindtext}
\usepackage{amssymb}
\usepackage{pifont}
\newcommand{\cmark}{\ding{51}}%
\newcommand{\xmark}{\ding{55}}%
\usepackage{amsthm}
\usepackage{caption}
\usepackage{mathtools}
\usepackage{epstopdf}
\AppendGraphicsExtensions{.tif}
\usepackage{url}
\newtheorem{theorem}{Theorem}[section]
\newtheorem{definition}[theorem]{Definition}
\newtheorem{lemma}{Lemma}[section]

\newtheorem{assumption}{Assumption}[section]
\usepackage{subcaption}

\usepackage{algorithm}
\usepackage{algpseudocode}
\usepackage{amsmath}

\usepackage{tikz}

\begin{document}

%
\title{Federated Learning Architectures: A Performance Evaluation with Crop Yield Prediction Application}

%

\author[myfirstdaddress,mysecondaddress]{Anwesha Mukherjee \corref{mycorrespondingauthor}}
\cortext[mycorrespondingauthor]{Corresponding author}\ead{anweshamukherjee2011@gmail.com}
\author[mysecondaddress]{Rajkumar Buyya} \ead{rbuyya@unimelb.edu.au}
\address[myfirstdaddress]{Department of Computer Science\\ Mahishadal Raj College, Mahishadal, West Bengal, 721628, India}
\address[mysecondaddress]{Cloud Computing and Distributed Systems (CLOUDS) Laboratory, \\School of Computing and Information Systems,\\ The University of Melbourne, Victoria, 3052, Australia}


%

%
\begin{abstract}
Federated learning has become an emerging technology for data analysis for IoT applications. This paper implements centralized and decentralized federated learning frameworks for crop yield prediction based on Long Short-Term Memory Network. For centralized federated learning, multiple clients and one server is considered, where the clients exchange their model updates with the server that works as the aggregator to build the global model. For the decentralized framework, a collaborative network is formed among the devices either using ring topology or using mesh topology. In this network, each device receives model updates from the neighbour devices, and performs aggregation to build the upgraded model. The performance of the centralized and decentralized federated learning frameworks are evaluated in terms of prediction accuracy, precision, recall, F1-Score, and training time. The experimental results present that $\geq$97\% and $>$97.5\% prediction accuracy are achieved using the centralized and decentralized federated learning-based frameworks respectively. The results also show that the using centralized federated learning the response time can be reduced by $\sim$75\% than the cloud-only framework. Finally, the future research directions of the use of federated learning in crop yield prediction are explored in this paper. 
\end{abstract}

%
%
\begin{keyword}
Federated learning \sep Prediction accuracy \sep Training time \sep Response time
\end{keyword}

%

%
\maketitle

\section{Introduction}
\label{intro}
Agriculture is an important sector that has a huge impact on the economy of most of the countries. The conventional farming practices depend on manual decision making on crop harvesting, irrigation, etc., that suffers from improper selection of crops, inefficient utilization of lands, water resources, etc. To overcome these problems, smart agriculture and farming practices come into the scenario, where Internet of Things (IoT) plays a significant role (\cite{debauche2021data, boursianis2020internet}). In IoT-based smart agricultural systems, IoT devices are used for soil and environmental parameters' data collection, and then analysing the data for decision making (\cite{bera2023cropreco, bera2023internet}). For data analysis, machine learning (ML) and deep learning (DL) are used, and the cloud servers are used for storing and analysing the large volume of data. However, data storage and analysis inside the cloud requires huge amount of IoT data transmission to the cloud, that requires seamless network connectivity and high network bandwidth. However, the agricultural lands are mainly located at rural areas, where the network bandwidth may not be high, as well as seamless network connectivity may not available. Hence, data transmission from IoT devices to the cloud is a challenge. Further, the entire data transmission from IoT devices to the cloud cause high network traffic, and the storage and analysis of the huge volume of data increase the cloud overhead, latency, etc. To address these issues, edge computing and fog computing come (\cite{bera2023cropreco, bera2024flag}). In edge computing, the resources are placed at the edge of the network to reduce the latency and computation overhead on the cloud. In fog computing, the intermediate devices, e.g. switch, router, etc., process data to reduce the latency and overhead on the cloud. Nevertheless, the concern regarding data security and privacy still remains. Moreover, the soil data, environmental data, climate and weather conditions of various geographical regions are different, and the user may not like to transmit and store the data over the cloud due to privacy protection. Hence, personalized local models and an efficient global model are required to perform accurate prediction and utilize the agricultural resources efficiently. To utilize the edge devices for local data analysis and for collaborative training to develop a global model, federated learning (FL) (\cite{nguyen2021federated, zhang2021survey, li2020review}) comes into the scenario. 
\par
FL is a learning approach that allows collaborative training with the coordination of multiple devices and a central server without sharing individual dataset (\cite{nguyen2021federated, mothukuri2021survey}). In an edge-cloud-based FL, the edge devices serve as the clients and the cloud server acts as the central server, for collaborative training (\cite{bera2024flag}). In an FL-based framework (\cite{nguyen2021federated}), the central server serves as the aggregator that at first creates an initial global model with learning parameters. Each of the clients downloads the current model from the server, computes own model updates using its local dataset, and offloads the local update to the server. The server receives local updates from all clients, and develops an improved global model. The clients download the global update from the server, compute their local updates again, and offloads the updates to the server. This process is continued until the global training is finished. The principal advantages of FL are enhanced data privacy, low-latency, and learning quality enhancement (\cite{nguyen2021federated}). As no data is shared, privacy is protected (\cite{mothukuri2021survey, djenouri2023federated}). Further, an enhanced version of the global model is created through collaborative training. As local data analysis takes place, the latency is reduced. 
\par
According to networking structure and data partitioning, the FL algorithms are classified into several categories. Based on data partitioning, FL algorithms are classified into three types (\cite{nguyen2021federated}): Vertical FL, Horizontal FL, and federated transfer learning. In vertical FL systems, the clients have datasets with same sample space but different feature space. In horizontal FL systems, the clients have datasets with same feature space but different sample space. In federated transfer learning, the clients have datasets with different feature space as well as different sample space. According to the networking structure, FL algorithms are divided into two categories: Centralized Federated Learning (CFL) and Decentralized Federated Learning (DFL). In a CFL system, all the clients train a model in parallel using their local datasets. Then, the clients send the trained parameters to the server. The server aggregates model parameters after receiving from all clients, and builds the updated global model. The clients get the updated model parameters from the server for the next training round. After the server finishes the global model training, each of the clients has same global model as well as its personalized local model. However, the communication with the server may not be always available. In such a case, DFL can be adopted. In DFL, all the clients form a collaborative network. For each of the communication rounds, the clients use their local datasets for local training. After that, each client performs model aggregation based on model updates received from the neighbour nodes. \par

\subsection{Motivation and Contributions}
Crop yield prediction is an important domain of smart agriculture, where ML is used for data analysis (\cite{van2020crop}). As farmers' information as well as soil and environmental parameters' data analysis take place, privacy is a major issue. In such a case, the use of conventional cloud-only framework for data analysis using ML raises concern regarding data privacy, latency, connectivity interruption due to poor network connectivity inside the agricultural lands, etc. To address these issues, the objective of the paper is to explore the use of FL for crop yield prediction. The major contributions of the paper are:
\begin{itemize}
    \item The use of FL in farming practices is discussed, and then an experimental case study is performed to analyse the performance of CFL and DFL in crop yield prediction. We consider a scenario where different number of devices perform collaborative training using CFL and DFL. Long Short-Term Memory (LSTM) Network is used as the underlying approach for both the CFL and DFL mechanisms.
    \item To implement CFL, a client-server paradigm is developed using socket programming, and multiple clients are handled by the server in the conducted experiment. For transmission of model updates \textit{MLSocket} is used. For aggregation, Federated Averaging (FedAvg) is used. The performance of CFL-based framework has been evaluated in terms of prediction accuracy, precision, recall, F1-Score, and training time. The experimental results present that the CFL-based framework has better prediction accuracy and lower response time than the cloud-only framework where the cloud server analyses the data after receiving from the client. 
    \item To implement the DFL, a collaborative network is formed using mesh topology or ring topology. In the conducted experiment, each node receives model updates from the neighbour nodes (the neighbour nodes depend on the selected topology), and performs aggregation to build the upgraded model. The performance of the nodes in both the topology are evaluated in terms of prediction accuracy, precision, recall, F1-Score, and training time. 
    \item Finally, the research challenges with CFL and DFL-based frameworks in crop yield prediction are highlighted in this paper.
\end{itemize}

\subsection{Layout of The Paper}
The rest of the paper is organized as follows: Section \ref{rel} briefly discusses the existing literature on FL, smart agriculture, and crop yield prediction. Section \ref{pro} illustrates the use of CFL and DFL in crop yield prediction. An experimental case study is presented in Section \ref{perf} on the use of CFL and DFL in crop yield prediction. Section \ref{future} explores the future research directions of FL in crop yield prediction. Finally, Section \ref{con} concludes the paper. \par
The list of acronyms used in this paper are listed in Table \ref{tab:acro}.

\begin{table}[]
    \centering
    \caption{Acronyms with full forms}
    \begin{tabular}{c|c}
    \hline
      Acronyms   &  Full form\\
      \hline
      ANN & Artificial Neural Network\\
      ML & Machine Learning\\
      DL & Deep Learning\\
      FL & Federated learning\\
      CFL & Centralized Federated Learning\\
      DFL & Decentralized Federated Learning\\
      IoT & Internet of Things\\
      IoAT & Internet of Agricultural Things\\
      LSTM & Long Short-Term Memory Network\\
      Bi-LSTM & Bidirectional Long Short-Term Memory Network\\
      FedAvg & Federated Averaging\\
      KNN & K-Nearest Neighbours\\
      DT & Decision Tree\\
      RF & Random Forest\\
      XGBoost & Extreme Gradient Boosting\\
      SVM & Support Vector Machine\\
      MLP & MultiLayer Perceptron\\
      LGBM & Light Gradient Boosting Machine\\
      GRU & Gated Recurrent Unit\\
      MLR & Multiple Linear Regression\\
      NB & Naive Bayes\\
      DNN & Deep Neural Network\\
      RNN & Recurrent Neural Network\\
      P2P & Peer-to-Peer\\
      FTL & Federated Transfer Learning\\
      \hline
    \end{tabular}
    \label{tab:acro}
\end{table}

\section{Related Work}
\label{rel}
Crop yield prediction and recommendation is a crucial area of IoT-based smart farming practices (\cite{debauche2021data, boursianis2020internet}). There are several research works carried out on the use of ML and DL in crop yield prediction. In (\cite{thilakarathne2022cloud}), the authors explored the use of several ML algorithms such as KNN, DT, RF, XGBoost, and SVM, for crop yield prediction. In (\cite{bakthavatchalam2022iot}), the authors used MLP neural network, decision table, and JRip for crop yield prediction. In (\cite{cruziot}), KNN was used for crop yield prediction. In (\cite{kathiria2023smart}), the authors used several ML approaches such as DT, SVM, KNN, LGBM and RF for crop yield prediction. LSTM, Bi-LSTM, and GRU-based framework were used in (\cite{gopi2024red}) for data analysis to predict crop yield. For crop yield prediction, MLR with ANN was used in (\cite{gopal2019novel}). In (\cite{dey2024fly}), Bi-LSTM was used for data analysis, and for better network connectivity the use of small cell with computation ability was explored. However, none of the existing approaches adopted FL in their frameworks. The major disadvantage of the conventional ML-based framework is compromise with data privacy as data sharing takes place with the cloud for analysis, requirement of high network bandwidth that may not be available at rural regions, high response time, high network traffic, huge overhead on the cloud server, etc. To address all these issues, FL can be adopted in crop yield prediction.
\par
The concept of FL relies on local data analysis, collaborative training, and generating global as well as personalized models. As no individual dataset is shared and a distributed learning is performed, data privacy is protected (\cite{nguyen2021federated}). Further, due to distributed nature, can be adopted in various IoT applications including healthcare, agriculture, transportation system, etc. The use of FL in IoT was elaborated in (\cite{nguyen2021federated}). The use of FL in fog computing environment was explored in (\cite{zhu2024flight}). In (\cite{atitallah2023fedmicro}), for distributed data analytics FL and transfer learning were adopted to propose an intelligent microservices-based framework for IoT applications. The use of FL in agriculture was explored in a few research works. In (\cite{bera2024flag}), CFL was used for soil health monitoring for irrigation decision making. The authors used CFL based on LSTM and DNN in their work. In (\cite{manoj2022federated}), FL was used for soybean yield prediction using deep residual network-based regression models for risk management in agricultural production. In (\cite{friha2022felids}), FL was used for intrusion detection in IoT-based agricultural systems. For yield forecasting FL was used in (\cite{li2024model}). For efficient data sharing in agri-food sector, the use of FL was discussed in (\cite{durrant2022role}). For crop classification, FL was used by (\cite{idoje2023federated}). The authors adopted CFL for crop classification based on Gaussian NB in (\cite{idoje2023federated}). As we observe, the use of FL in crop yield prediction was explored in a few existing works, and most of them focused on the use of CFL. In this paper, we provide an experimental study of both the CFL and DFL in crop yield prediction based on LSTM. 
\par
In Table \ref{tab:comsur}, the existing works have been compared with respect to the proposed FL-based framework for crop yield prediction. As we observe from the table, most of the existing works rely on conventional ML/DL-based framework, and compared to the existing CFL-based framework, this work explores the use of both CFL and DFL (with mesh as well as ring-based networks) in crop yield prediction, and determines the training time as well as response time. 

\begin{sidewaystable}
\caption{\centering{Comparison of our work with existing crop yield prediction frameworks}}
\small
    \centering
    \begin{tabular}{|c|c|c|c|c|c|c|}
        \hline
        \textbf{Work} &	\textbf{Classifier} &  \textbf{Multi-crop} & \textbf{CFL} & \textbf{DFL}	& \textbf{Training} & \textbf{Response}\\
        & & \textbf{dataset}& \textbf{is used} & \textbf{is used}&  \textbf{time is} &  \textbf{time is}\\
         & & \textbf{is used}& &  &\textbf{measured}&\textbf{measured}\\
        \hline
        \cite{thilakarathne2022cloud} & RF, DT, KNN, & \cmark & \xmark & \xmark & \xmark& \xmark\\ 
& XGBoost, SVM	 & & & & &\\
        \hline
       \cite{bakthavatchalam2022iot}& MLP, Decision Table, & \cmark	 & \xmark & \xmark	 & Measured model & \xmark \\
  & JRip & & & &build time&\\
        \hline
       \cite{cruziot} & KNN &\cmark & \xmark & \xmark	 & \xmark  & \xmark\\
   &  & & & &&\\
  \hline
    \cite{kathiria2023smart} & DT, SVM, KNN, & \cmark & \xmark & \xmark & \xmark & \xmark\\
   & LGBM, RF & & & & &\\
   \hline
       \cite{gopi2024red} & LSTM, Bi-LSTM, & \cmark& \xmark & \xmark & \xmark & \xmark\\
    & GRU & & & & &\\
        \hline 
        \cite{idoje2023federated} & Gaussian NB & \cmark& \cmark & \xmark & \xmark & \xmark\\
    & & & & & &\\
        \hline 
        Our work   & LSTM & \cmark & \cmark & \cmark& \cmark & \cmark\\
        \hline 
    \end{tabular}
    \label{tab:comsur}
\end{sidewaystable}

\section{Federated Learning in Crop Yield Prediction}
\label{pro}
In Section \ref{intro}, we have briefly discussed on CFL and DFL. For crop yield prediction both the approaches can be used. In an IoT-based crop yield prediction framework, the IoT devices collect data of soil and environmental parameters such as temperature, humidity, rainfall, soil pH, Nitrogen, Phosphorous, Potassium level, etc. The collected IoT data is processed inside the cloud servers. However, for data privacy protection, to reduce latency and deal with poor network connectivity inside the rural regions containing the agricultural lands, FL is adopted in IoAT. The integration of FL with IoT-based crop yield prediction framework permits the local data analysis inside the edge devices that can work as the clients. For data analysis we have used LSTM in this work. To capture the temporal dependencies and retain the sequential nature of the soil and environmental data, LSTM is considered. The mathematical notations used in this work are defined in Table \ref{mathnote}. 

\begin{table}[]
    \centering
    \caption{Mathematical notations with definitions}
    \begin{tabular}{c|c}
    \hline
        Notations &  Definition\\
    \hline
       $\phi_t$  & Forget gate\\
       $\zeta_t$  & Input gate\\
       $\lambda_t$  & Output gate\\
       $weight_\phi$ & Weight matrix of forget gate\\
       $weight_\zeta$ & Weight matrix of input gate\\
       $weight_\lambda$ & Weight matrix of output gate\\
       $bias_\phi$ & Bias for forget gate\\
       $bias_\zeta$ & Bias for input gate\\
       $bias_\lambda$ & Bias for output gate\\
       $h_t$ & Present hidden state\\
       $h_{t-1}$ & Previous hidden state\\
       $x_t$ & Current input\\
       $\mathcal{\hat{C}}_t$ & Candidate value\\
       $C_t$ & Cell state\\
       $C_{t-1}$ & Previous cell state\\
       $Data_c$ & Data of client $c$\\
       $B$ & Batch size\\
       $\eta$ & Learning rate\\
       $N_b$ & Number of batches\\
       $N_e$ & Number of epochs\\
       $N_r$ & Number of rounds\\
       $m_c$ & Model updates of client $c$\\
       $N_c$ & Number of connected clients\\
       $f_c$ & Fraction of clients participating in CFL\\
       $m_s$ & Model parameter of the server\\
       $m_{final}$ & Final global model update\\
       $m_0$ & Initial model parameters of a node in DFL\\
       $m_p$ & Model updates of node $p$ in DFL\\
       $P$ & Set of nodes in the DFL framework\\
       $N_p$ & Number of neighbours nodes in DFL\\
       $\mathcal{L}$ & Loss function\\
       $\alpha$ & True positive\\
       $\beta$ & True negative\\
       $\gamma$ & False positive\\
       $\rho$ & False negative\\
       $\mathcal{A}$ & Accuracy\\
       $\mathcal{P}$ & Precision\\
       $\mathcal{R}$ & Recall\\
       $\mathcal{F}$ & F1-Score\\
       \hline
    \end{tabular}
    \label{mathnote}
\end{table}

\par
LSTM is an upgraded version of RNN that considers a memory cell which is controlled by input gate, forget gate, and output gate. LSTM maintains a chain-like structure that has four neural networks and different memory blocks referred to as cells. To learn long-term dependencies the gates play significant roles by retaining or discarding information in a selective manner. For the short-term memory a hidden state is maintained by the LSTM network, that is updated depending on the previous hidden state, input, and the current state of the memory cell. \\
In LSTM, the forget gate controls which information will be removed from the memory cell, and mathematically expressed as:
\begin{equation}
    \phi_t=\sigma(weight_\phi \dots [h_{t-1}, x_t] + bias_\phi)
\end{equation}
where $\sigma$ denotes the sigmoid function.\\
The input gate is used to add information to the memory cell, and mathematically expressed as:
\begin{equation}
    \zeta_t=\sigma(weight_\zeta \dots [h_{t-1}, x_t] + bias_\zeta)
\end{equation}
where, a sigmoid function is used to regulate the information and filter the values to retain. After that, $tanh$ function is used to create a vector having all possible values from $h_{t-1}$ and $x_t$, as follows:
\begin{equation}
    \mathcal{\hat{C}}_t=tanh(weight_{C} \dots [h_{t-1}, x_t] + bias_C)
\end{equation}
where $weight_{C}$ denotes the weight matrix and $bias_C$ denotes the bias.\\
Finally, the regulated values are multiplied with the values of the vector to get the information to be added to the memory cell, as follows:
\begin{equation}
    \mathcal{C}_t=\phi_t \odot C_{t-1} + \zeta_t \odot \mathcal{\hat{C}}_t
\end{equation}
The output gate that extracts the useful information from the current memory cell state as the output, is mathematically expressed as follows:
\begin{equation}
    \lambda_t=\sigma(weight_\lambda \dots [h_{t-1}, x_t] + bias_\lambda)
\end{equation}
Firstly, using $tanh$ function a vector is generated. After that, a sigmoid function is used to regulate the information and filter the values to retain using the inputs $h_{t-1}$ and $x_t$. Finally, the regulated values are multiplied with the values of the vector to be sent as an output as well as input to the next cell. As LSTM is able to capture long-term dependencies, LSTM performs well in sequence prediction tasks, time series, etc. 

\subsection{CFL-based framework}
In our CFL-based system, all the clients get the initial model parameters from the server, train their individual models using their local datasets, and then transmit the model updates to the server node. The server node works as the aggregator that receives model updates from all the clients and performs aggregation to update the global model accordingly. Here, for aggregation, we use FedAvg. The updated global model is sent to the participating clients. Hence, at the end of the process each client has its personalized local model and the global model. The client and server-side algorithms are stated in Algorithm \ref{algo_1} and Algorithm \ref{algo_2} respectively. In the algorithms, $c$ represents a client and $S$ represents the server. Algorithm \ref{algo_1} presents the steps of the client-side process, where each connected client gets model parameters from the server, trains the local model using its local dataset, and returns the model updates to the server. Algorithm \ref{algo_2} presents the steps of the server-side process, where the server receives model updates from the connected clients, performs aggregation, and updates the model accordingly. The pictorial representation of CFL is presented in Fig. \ref{fig:cfl}, where $N_c$ clients participate in a collaborative training process with the server that works as the aggregator to aggregate the model updates received from the clients to build the global model.  

\begin{figure*}
    \centering
    \includegraphics[width=0.99\linewidth, height=2.7in]{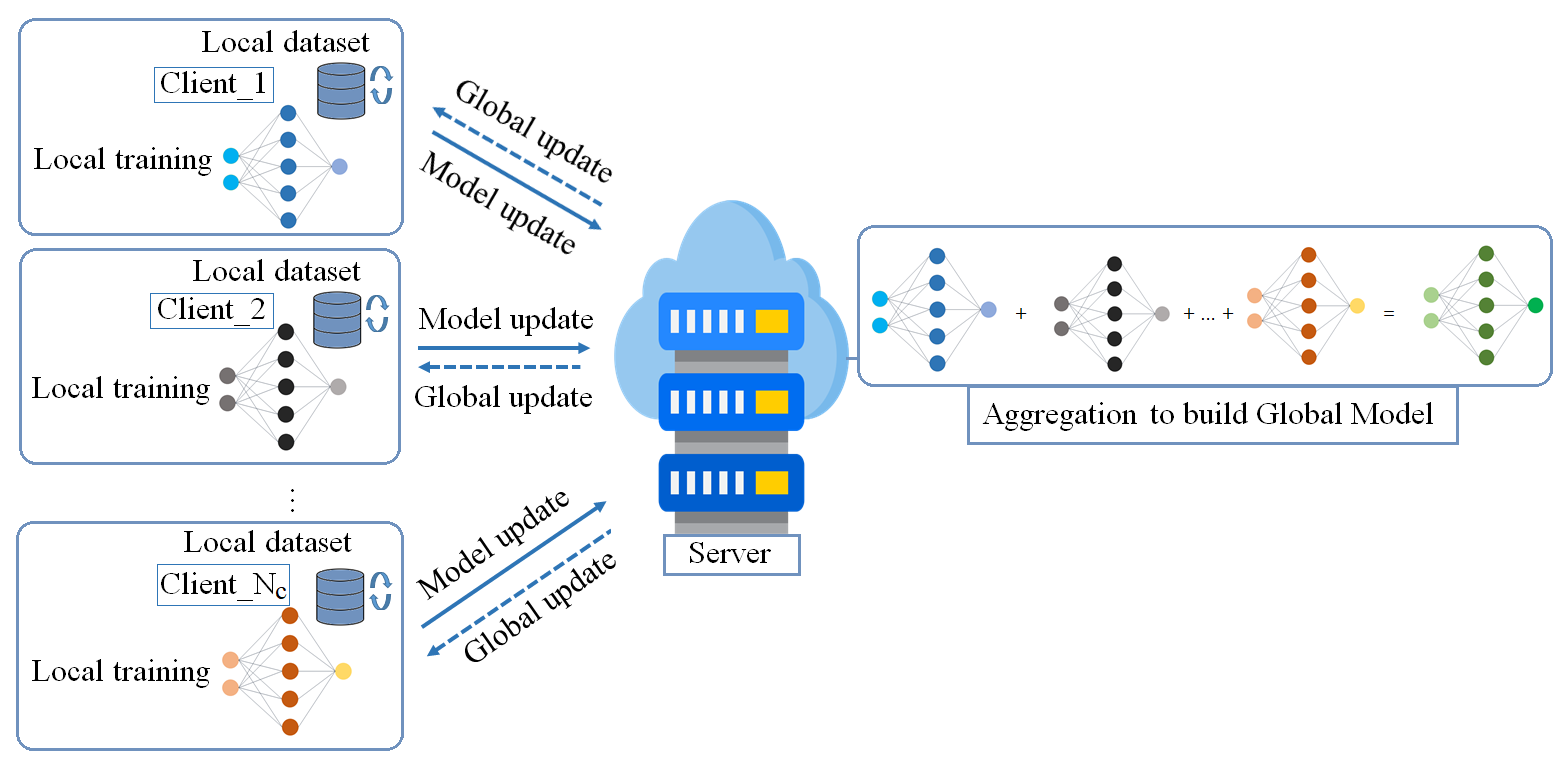}
    \caption{\centering{The CFL process}}
    \label{fig:cfl}
\end{figure*}

\begin{algorithm} 
\caption{Client-side algorithm}
\label{algo_1}
\begin{algorithmic}[1]
  \renewcommand{\algorithmicrequire}{\textbf{Input:}}
   \renewcommand{\algorithmicensure}{\textbf{Output:}}
  \Require $Data_c$, $N_b$, $N_e$
  \Ensure $m_{c}$\\
  \textbf{$Function\;Update(Data_c, N_b, N_e)$}:
  \State{$PData_c \gets Preprocess(Data_c)$} 
  \While{$connected \; with \; S$}
       \State{$Train(m_{c} \gets getmodelparameters())$}
    \EndWhile
    \State{$SaveParameters(m_{c})$}\\
   \textbf{$Function\;Train(m_{c})$}:
      \State{$N_b \gets Split(PData_c, B)$}    \Comment{split data into $N_b$ batches}
      \For{$e=0 \; to \; N_e-1$} 
           \For{$b=1\; to \; N_b$}
              \State{$m_c^{e+1} \gets m_c^{e}-\eta\nabla m_c^e$} \Comment{$\nabla m_c^e$ represents the gradient}
           \EndFor
    \EndFor
    \State{$m_c \gets m_c^{N_e}$}
    \State{$sendmodelupdate(m_c)$} \Comment{send model update to $S$}
\end{algorithmic}
\end{algorithm}

\begin{algorithm} 
\caption{Server-side Algorithm}
\label{algo_2}
\begin{algorithmic}[1]
  \renewcommand{\algorithmicrequire}{\textbf{Input:}}
   \renewcommand{\algorithmicensure}{\textbf{Output:}}
  \Require $N_c$, $f_c$, $N_r$
  \Ensure $m_{final}$\\
  \textbf{$Function\;Collect(N_c, N_r)$}:
  \State{$ConnectedClients \gets []$}
  \While{$(length(ConnectedClients) \neq N_c)$}
       \State{$listen()$}
       \State{$acceptconnection()$}
    \EndWhile
    \State{$FedAvg()$}
    \State{Release clients}  \\
   \textbf{$Function\;FedAvg()$}:
   \State{$m_s^0 \gets InitModel()$} \Comment{initial model is generated}
      \For{$r =1 \; to \; N_r$}
              \State{$M_r \gets Subset(max(f_c*N_c,1),``random")$}
              \State{$MU \; \gets \; []$}
             \For{$c \in M_r$}
             \State{$m_{c}^r \gets getmodelupdate(c)$} \Comment{get model update from client $c$ at round $r$}
                   \State{$MU.append(m_{c}^r)$}
                   \EndFor
                   \State{$m_s^{r+1} \gets \frac{1}{N_c} \sum_{c=1}^{N_c} \cdot {m_{c}^r}$}
                   \State{$sendtoclients(m_s^{r+1})$}
    \EndFor
    \State{$m_{final} \gets m_s^{N_r+1}$}
\end{algorithmic}
\end{algorithm}

\par
In CFL, the server is the aggregator and it distributes the model updates with the clients. The clients have their personalized models along with the global model update. Each of the clients can perform data analysis locally through a collaborative training process without sharing data. Hence, privacy is protected as well as through collaborative training prediction accuracy is enhanced. The time complexity of the CFL process depends on the time complexity of model initialization, local model training, exchange of model updates, and aggregation. The time complexity of model initialization is given as $O(1)$. The time complexity of local model training is given as $O(N_r \cdot N_e \cdot N_b \cdot m_c)$. The time complexity of exchanging model updates is given as $O(N_r \cdot N_c \cdot (m_c+m_s))$. The time complexity of aggregation is given as $O(N_r \cdot N_c \cdot m_c)$.  
\par
Though, there are several benefits, the CFL has some limitations. As the server performs as the aggregator, good network connectivity with the server is highly desirable. However, many applications do not have the provision of seamless network connectivity. Further, the overhead on the server is very high because the aggregation takes place inside the server. Further, sharing model updates by all the clients with the server may raise a concern regarding security. To address these limitations, DFL has come. 
\subsection{DFL-based framework}
In case of crop yield prediction, the data collection takes place at the rural regions, where the network connectivity is usually poor. In that case, the communication with the cloud server is a major issue. Therefore, if the network connectivity is poor, DFL can be used by the edge devices for collaborative training purpose. In DFL, the clients form a network among themselves and perform collaborative learning. Here, each node is a learner as well as contributor. In our work, we have considered two types of DFL frameworks where the clients form a network either using ring or mesh topology. The ring-based network is referred to as P2P network also, where each peer exchanges its model updates with two neighbour nodes. In case of the mesh-based network, each node exchanges its model updates with rest of the nodes in that network. The DFL process for ring-based and mesh-based networks are stated in Algorithm \ref{algo_3} and Algorithm \ref{algo_4} respectively, where $p$ denotes a node and $P$ denotes the set of nodes in the formed network. The pictorial representation of DFL using ring-based and mesh-based networks are presented in Figs. \ref{fig:ring} and \ref{fig:mesh}, where four nodes form a network using ring topology and mesh topology respectively.

\begin{figure*}
    \centering
\begin{minipage}{0.495\linewidth}
\includegraphics[width=0.99\linewidth, height=2.7in]{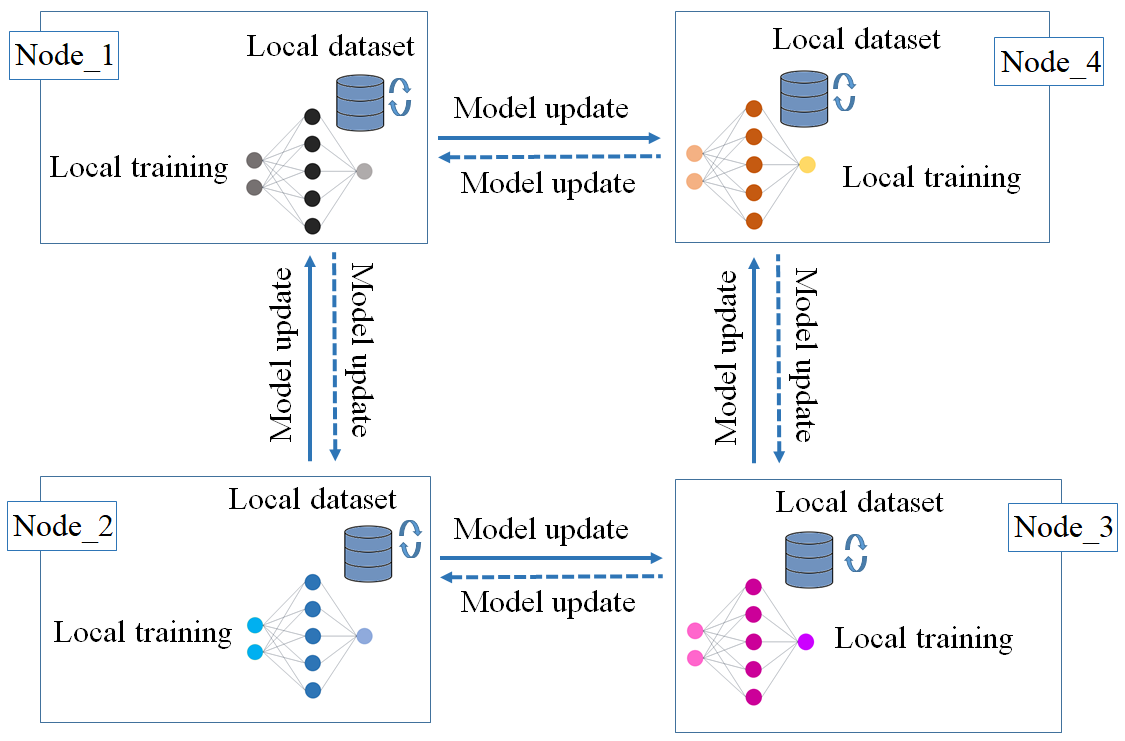}
\caption{\centering{DFL using ring-based network}}
\label{fig:ring}
\end{minipage}
\begin{minipage}{0.495\linewidth}
\includegraphics[width=0.99\linewidth, height=2.7in]{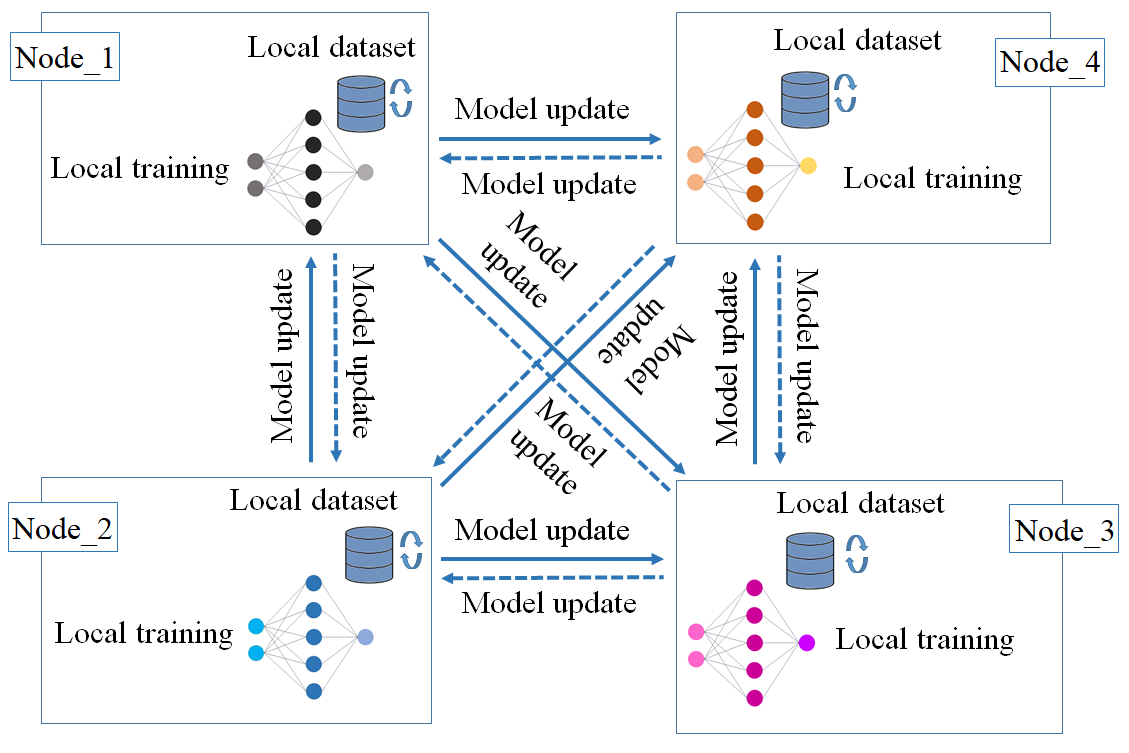}
\caption{\centering{DFL using mesh-based network}}
\label{fig:mesh}
\end{minipage}
\end{figure*}

\begin{algorithm}
\caption{Model update in Ring-based P2P Network}
\label{algo_3}
\begin{algorithmic}[1]
\renewcommand{\algorithmicrequire}
{\textbf{Input:}}\renewcommand{\algorithmicensure}
{\textbf{Output:}}\Require $P$, $N_p$, $m_0$, $N_r$
\Ensure $m_{p}, \forall p \in P$ \\
\textbf{$Function \; Training(N_p, m_0)$}:
\For{$p =1 \; to \; |P|$}
\State $m_p \gets m_0$ \Comment{Initialize model parameters}
\EndFor
\For{$r=1 \; to \; N_r$}
\For{$p = 1 \; to \; |P|$}   \Comment{Local model training}
\State{$PData_p \gets Preprocess(Data_p)$}
\State{$N_b \gets Split(PData_p, B)$}    \Comment{split data into $N_b$ batches}
      \For{$e=0 \; to \; N_e-1$} 
           \For{$b=1 \; to \; N_b$}
           \State{$m_p^{e+1} \gets m_p^{e}-\eta\nabla m_p^e$} \Comment{$\nabla m_p^e$ represents the gradient}
           \EndFor
    \EndFor 
\EndFor
\For{$p = 1 \; to \; |P|$}  \Comment{Exchange of model updates and aggregation}
\State{$M_{recv} \gets []$}
    \State Send $m_p$ to $p_{pre}$ and $p_{suc}$
    \State $m_j^r \gets getmodelupdate(j)$, where $j \in \{p_{pre}, p_{suc}\}$ 
    \State{$M_{recv}.append(m_j^r)$}
    \State $m_{p}^{r+1} \gets \sum_{j=1}^{N_p} m_j^r/N_p$  \Comment{Aggregate model updates received from $p_{pre}$ and $p_{suc}$}
\EndFor
\EndFor
\end{algorithmic}
\end{algorithm}

\begin{algorithm}
\caption{Model update in Mesh-based Network}
\label{algo_4}
\begin{algorithmic}[1]
\renewcommand{\algorithmicrequire}
{\textbf{Input:}}\renewcommand{\algorithmicensure}
{\textbf{Output:}}\Require $Data_p$, $P$, $N_p$, $m_0$, $N_r$
\Ensure $m_{p}, \forall p \in P$ \\
\textbf{$Function \; Training(N_p, m_0)$}:
\For{$p =1 \; to \; |P|$}
\State $m_p \gets m_0$ \Comment{Initialize model parameters}
\EndFor
\For{$r=1 \; to \; N_r$}
\For{$p = 1 \; to \; |P|$}   \Comment{Local model training}
\State{$PData_p \gets Preprocess(Data_p)$}
\State{$N_b \gets Split(PData_p, B)$}    \Comment{split data into $N_b$ batches}
      \For{$e =0 \; to \; N_e-1$} 
           \For{$b=1 \; to \; N_b$}
           \State{$m_p^{e+1} \gets m_p^{e}-\eta\nabla m_p^e$} \Comment{$\nabla m_p^e$ represents the gradient}
           \EndFor
    \EndFor 
\EndFor
\For{$p = 1 \; to \; |P|$}  \Comment{Exchange of model updates and aggregation}
\State{$M_{recv} \gets []$}
\For{$j = 1 \; to \; N_p$}
    \State Send $m_p$ to $j$, where $j \in (P-p)$ 
    \State $m_j^r \gets getmodelupdate(j)$, where $j \in (P-p)$ 
    \State{$M_{recv}.append(m_j^r)$}
    \EndFor
    \State $m_{p}^{r+1} \gets \sum_{j=1}^{N_p} m_j^r/N_p$ \Comment{Aggregate model updates received from all other nodes}
\EndFor
\EndFor
\end{algorithmic}
\end{algorithm}
We observe from Algorithm \ref{algo_3} that in the ring-based DFL approach, a node $p$ sends and receives model updates to and from its its preceding node ($p_{pre}$) and the succeeding node ($p_{suc}$). In the mesh-based DFL approach as presented in Algorithm \ref{algo_4}, a node $p$ sends and receives model updates to and from rest of the nodes of the network. In Algorithms \ref{algo_3} and \ref{algo_4}, we have initially considered an empty array $M_{recv}$. As the model updates are received from the neighbour nodes, the received updates are appended to store all the received updates inside $M_{recv}$. Finally, aggregation takes place based on the received updates from the nodes in both the ring-based and mesh-based networks. The time complexity of the DFL-based framework depends on the time complexity of model initialization, local model training, exchanging model updates, and aggregation. The time complexity of model initialization is given as $O(1)$. The time complexity of local model training is given as $O(N_r \cdot N_e \cdot N_b \cdot m_p)$. The time complexity of exchanging model updates is given as $O(N_r \cdot N_p \cdot (m_p+m_j))$. The time complexity of aggregation is given as $O(N_r \cdot N_p \cdot m_j)$, where $1 \leq j \leq N_p$.
\par
If a DFL framework contains three nodes, then the number of exchange of model updates will be same for both the mesh and ring-based networks. However, for the $number of nodes>=4$, the results will be different as the number of model updates exchange differ for the mesh and ring-based networks. 
\subsection{Proof of Convergence in FL}
\label{convergence}
For both the CFL and DFL approaches, the final objective is to minimize the global loss function $\mathcal{L}(m)$ that is mathematically defined as follows:
\begin{align}
\mathcal{L}(m) = \frac{1}{K} \sum_{i=1}^K \mathcal{L}_i(m)
\end{align}
where $K$ is the number of participating nodes, i.e. $K=N_c$ for CFL and $K=N_p$ for DFL, and $\mathcal{L}_i(m)$ is the local loss function at node $i$. 
\begin{definition}
\textit{Lipschitz Continuity}: $\mathcal{L}_i(m)$ is $\mathcal{L}$-Lipschitz continuous if there exists a constant $C_1 > 0$ such that $\forall  m, \mathbf{w}$, we have 
\begin{align}
\mathcal{L}_i(m) \leq \mathcal{L}_p(\mathbf{w}) + \nabla \mathcal{L}_p(\mathbf{w})^T (m - \mathbf{w}) + \frac{C_1}{2} \|m - \mathbf{w}\|^2  
\end{align}
\end{definition}
\begin{definition}
\textit{Bounded Variance}: The variance of the stochastic gradients is bounded if there exists a constant $C_2^2 > 0$ such that $\forall m$:
\begin{align}
\mathbb{E}[\|\nabla \mathcal{L}_i(m) - \nabla \mathcal{L}(m)\|^2] \leq C_2^2
\end{align}
\end{definition}
\begin{definition}
\textit{Unbiased Gradients}: The stochastic gradients are unbiased estimates of the true gradients if $\forall m$:
\begin{align}
\mathbb{E}[\nabla \mathcal{L}_i(m)] = \nabla \mathcal{L}(m)  
\end{align}
\end{definition}
\begin{assumption}
\textit{Smoothness}: The global loss function $\mathcal{L}(m)$ is smooth, i.e., there exists a constant $C_3 > 0$, such that $\forall m, \mathbf{w}$:
\begin{align}
\mathcal{L}(m) \leq \mathcal{L}(\mathbf{w}) + \nabla \mathcal{L}(\mathbf{w})^T (m - \mathbf{w}) + \frac{C_3}{2} \|m - \mathbf{w}\|^2  
\end{align}
\end{assumption}
\begin{assumption}
\textit{Strong Convexity}: The global loss function $\mathcal{L}(m)$ is strongly convex, i.e., there exists a constant $C_4 > 0$ such that $\forall  m, \mathbf{w}$:
\begin{align}
\mathcal{L}(m) \geq \mathcal{L}(\mathbf{w}) + \nabla \mathcal{L}(\mathbf{w})^T (m - \mathbf{w}) + \frac{C_4}{2} \|m - \mathbf{w}\|^2
\end{align}
\end{assumption}
\begin{lemma}
\textit{Gradient Bound:}
Under the assumptions of Lipschitz continuity and bounded variance, the gradient of $\mathcal{L}(m)$ is bounded:
\begin{align}
\mathbb{E}[\|\nabla \mathcal{L}(m)\|^2] \leq \frac{C_1}{K} \sum_{i=1}^K \|\nabla \mathcal{L}_i(m)\|^2 + \frac{C_2^2}{K}
\end{align}
\begin{proof}
By Lipschitz continuity of $\mathcal{L}_i(m)$:
\begin{align}
\|\nabla \mathcal{L}_i(m)\|^2 \leq C_1^2 \|m\|^2
\end{align}
Taking the expectation and summing over all nodes:
\begin{align}
\mathbb{E}[\|\nabla \mathcal{L}(m)\|^2] = \frac{1}{K^2} \sum_{i=1}^K \mathbb{E}[\|\nabla \mathcal{L}_i(m)\|^2] \leq \frac{C_1^2}{K^2} \sum_{i=1}^K \|m\|^2 + \frac{C_2^2}{K}
\end{align}
Thus,
\begin{align}
   \mathbb{E}[\|\nabla \mathcal{L}(m)\|^2] \leq \frac{C_1}{K} \sum_{i=1}^K \|\nabla \mathcal{L}_i(m)\|^2 + \frac{C_2^2}{K}
\end{align}
\end{proof}
\end{lemma}
\begin{theorem}
Under the Lipschitz continuity, bounded variance, unbiased gradients, smoothness, and strong convexity, the FL-based framework converges to a stationary point of $\mathcal{L}(m)$.
\end{theorem}
\begin{proof}
\textit{Step 1: Local Update Rule}: Each node $i$ performs local updates using stochastic gradient descent for $N_e$ local epochs. Let $m_i^{e}$ denotes the model parameters at node $i$ at epoch $e$, then
\begin{align}
m_i^{e+1} = m_i^{e} - \eta\nabla m_i^{e}
\end{align}
\textit{Step 2: Aggregation}: After $N_e$ epochs, each node exchanges its model updates with other nodes in DFL, and with the server in CFL. For  CFL, the aggregation takes place at the server. After receiving the model update, the client trains with local dataset, and sends the updates to the server. \\
For CFL, the global model update at round $r$ is defined as:
\begin{align}
m^{r+1} = \frac{1}{K} \sum_{i=1}^{K} m_i^{r}
\end{align}
where $K=N_c$.\\
For DFL, each node aggregates the updates after receiving from neighbour nodes ($N_p$). Hence, the local model update at round $r$ is defined as:
\begin{align}
m_i^{r+1} = \frac{1}{N_p} \sum_{j=1}^{N_p} m_j^{r}
\end{align}
The global model update at round $r$ is defined as:
\begin{align}
m^{r+1} = \frac{1}{K} \sum_{i=1}^{K} m_i^{r+1}
\end{align}
where $K=N_p$.\\
\textit{Step 3: Bounding the Global Loss}: Using the smoothness assumption, the change in $\mathcal{L}(m)$ is expressed as follows:
\begin{align}
\mathcal{L}(m^{r+1}) \leq \mathcal{L}(m^r) + \nabla \mathcal{L}(m^r)^{N_r} {(m^{r+1} - m^r)}+ \frac{C_2}{2} \|m^{r+1} - m^r\|^2 
\end{align}
Taking the expectation over the stochastic gradients, we obtain
\begin{align}
\mathbb{E}[\mathcal{L}(m^{r+1})] \leq \mathcal{L}(m^r) - \frac{\eta}{K} \|\nabla \mathcal{L}(m^r)\|^2 + \frac{C_3\eta^2 C_2^2}{2K}   
\end{align}
Now, summing this inequality over $N_r$ rounds, we get
\begin{align}
\sum_{r=1}^{N_r} \mathbb{E}[\mathcal{L}(m^{r+1}) - \mathcal{L}(m^r)] \leq - \frac{\eta}{K} \sum_{r=1}^{N_r} \|\nabla \mathcal{L}(m^r)\|^2 + \frac{C_3 \eta^2 C_2^2 {N_r}}{2K}
\end{align}
After rearranging terms, we find
\begin{align}
\frac{1}{N_r} \sum_{r=1}^{N_r} \mathbb{E}[\|\nabla \mathcal{L}(m^r)\|^2] \leq \frac{\mathcal{L}(m^1) - \mathcal{L}(m^{N_r+1})}{\eta N_r} + \frac{C_3 \eta C_2^2}{2K}
\end{align}
As $N_r \to \infty$, the term $\frac{\mathcal{L}(m^1) - \mathcal{L}(m^{N_r+1})}{\eta N_r}$ approaches zero, ensuring the following
\begin{align}
\lim_{N_r \to \infty} \frac{1}{N_r} \sum_{r=1}^{N_r} \mathbb{E}[\|\nabla \mathcal{L}(m^r)\|^2] = 0
\end{align}
This demonstrates that the global model updates $m$ converge to a stationary point of $\mathcal{L}(m)$.
\end{proof}
\subsection{Performance metrics}
In the next section, we have analysed the performance of both the CFL and DFL-based frameworks in crop yield prediction in terms of prediction accuracy, precision, recall, F1-Score, and training time.  
\par
The accuracy of a model is determined as follows:
\begin{equation}
    \mathcal{A}=\frac{\alpha+\beta}{\alpha+\beta+\gamma+\rho}
\end{equation}
\par
The precision of a model is determined as follows:
\begin{equation}
    \mathcal{P}=\frac{\alpha}{\alpha+\gamma}
\end{equation}
\par
The recall of a model is determined as follows:
\begin{equation}
    \mathcal{R}=\frac{\alpha}{\alpha+\rho}
\end{equation}
\par
The F1-Score of a model is determined as follows:
\begin{equation}
    \mathcal{F}=2*\frac{\mathcal{P}*\mathcal{R}}{\mathcal{P}+\mathcal{R}}
\end{equation}
\par
We have already discussed the time complexity of CFL and DFL, where we observe that the time consumption depends on the time consumed for model initialization, training, exchange of model updates, and aggregation. 
\par
Therefore, in the CFL-based approach, the training time of the server is determined as the sum of the time consumption for model initialization ($T_{init_1}$), model training ($T_{train_1}$), exchanging updates with the clients ($T_{ex_1}$), and aggregation ($T_{agg_1}$), given as,
\begin{equation}
    T_{CFL}=T_{init_1}+T_{train_1}+T_{ex_1}+T_{agg_1}
\end{equation}
\par
In the DFL-based framework, the training time of a node is determined as the sum of the time consumption for model initialization ($T_{init_2}$), model training ($T_{train_2}$), exchanging updates with the neighbour nodes ($T_{ex_2}$), and aggregation ($T_{agg_2}$), given as, 
\begin{equation}
    T_{DFL}=T_{init_2}+T_{train_2}+T_{ex_2}+T_{agg_2}
\end{equation}
\section{Performance Evaluation}
\label{perf}
This section describes the implementation, experimental setup used for analysis, and then analyses the performance of CFL and DFL based on the experimental setup.  

\subsection{Implementation}
A design and implementation of the architecture of CFL shown in Fig. \ref{fig:cfl} is presented in Fig. \ref{fig:implement} along with detailing interaction between FL clients and the global server. In Figs. \ref{fig:ring} and \ref{fig:mesh}, we have presented the architecture of DFL using ring and mesh topology respectively. Here, we present the implementation diagrams of DFL using ring and mesh topology with four nodes in Figs. \ref{ringimp} and \ref{meshimp}, respectively. For implementation, we use Python language. \textit{Tensorflow} is used along with LSTM supported by it. To build the client-server model and communication over the network, socket programming is used. For the transfer of model updates,  we used \textit{MLSocket}. For secure communication, Secure Shell protocol is used. The considered dataset is split, and assigned to the clients as the local datasets, and to the server as the global dataset. As the first and second dense layer activation function \textit{ReLU} is used. \textit{Softmax} is used as the Output layer activation function. \textit{Sparse Categorical Crossentropy} is used as the Loss function, and as the optimizer \textit{Adam} is used. 
\begin{figure*}
    \centering
    \includegraphics[width=0.99\linewidth, height=3.2in]{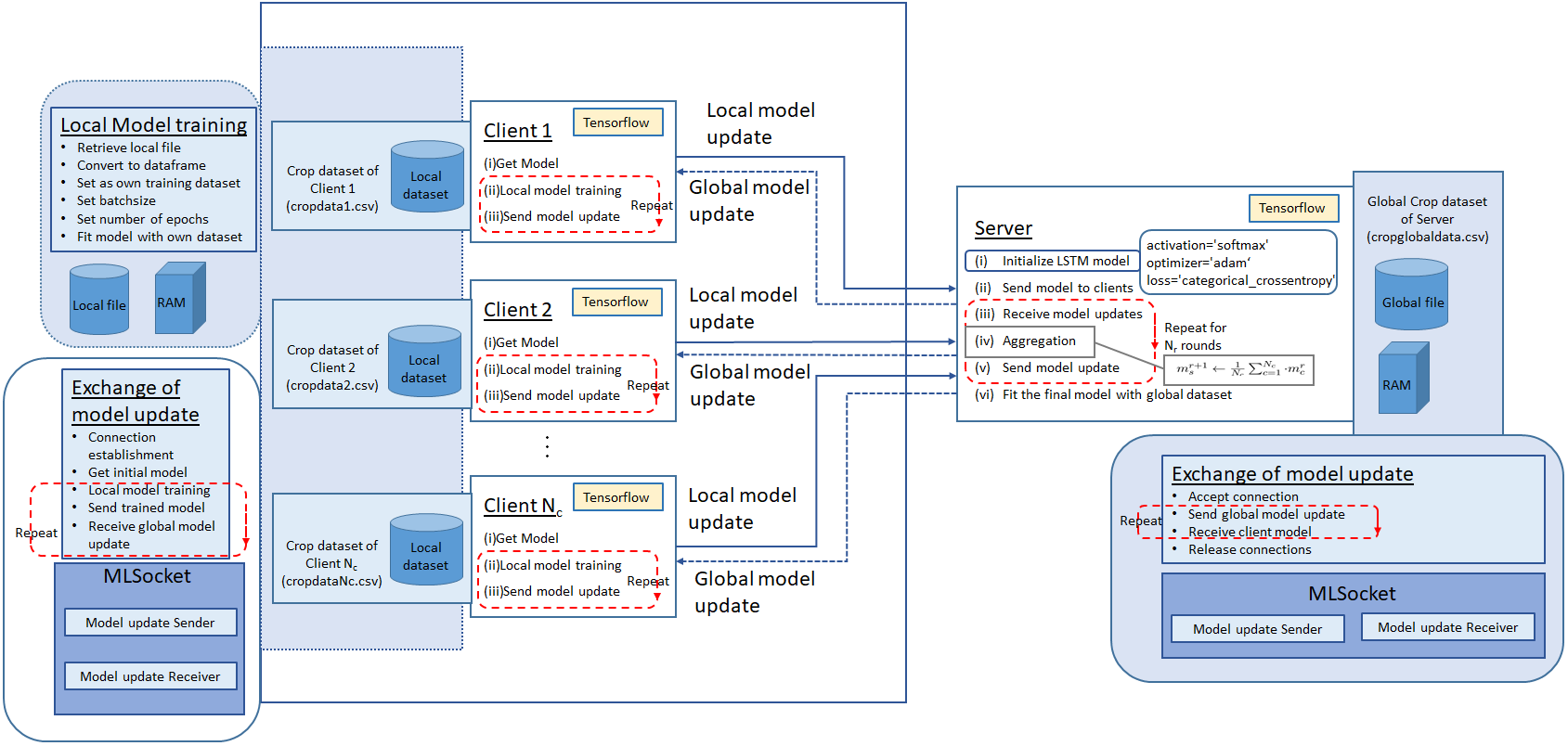}
    \caption{\centering{The experimental implementation diagram of CFL-based crop yield prediction}}
    \label{fig:implement}
\end{figure*}
\subsection{Experimental setup}
The experiment was conducted in the CLOUDS lab, The University of Melbourne. We created \textit{sixteen instances (H1 to H16)} over \textit{RONIN Cloud Platform}. The configuration of each instance is:
\begin{itemize}
    \item 4GB RAM
    \item 2 vCPUs
    \item 100GB SSD
\end{itemize}
Among these instances \textit{one instance} (H3) was selected as the \textit{server machine}, and other \textit{fifteen instances} (H1, H2, H4-H15) served as the \textit{client machines}. The performance of CFL and DFL in crop yield prediction were analysed in terms of prediction accuracy and training time. For experimental analysis we used the dataset{\footnote{\url{https://www.kaggle.com/datasets/atharvaingle/crop-recommendation-dataset}}}, containing 2200 samples of 22 different classes (rice, maize, pigeonpeas, chickpea, mothbeans, mungbean, kidneybeans, blackgram, lentil, jute, grapes, pomegranate, watermelon, mango, muskmelon, orange, papaya, banana, apple, coconut, cotton, coffee). There are seven features: Nitrogen (N), Phosphorous (P), Potassium (K), temperature, pH, humidity, and rainfall, which were considered as the input, and the crop was considered as the output. The learning rate was considered 0.001. The number of local epochs was considered 100 in both CFL and DFL. 
\begin{figure*}
    \centering
    \begin{minipage} {0.49\linewidth}
    \includegraphics[width=0.99\linewidth, height=3.2in]{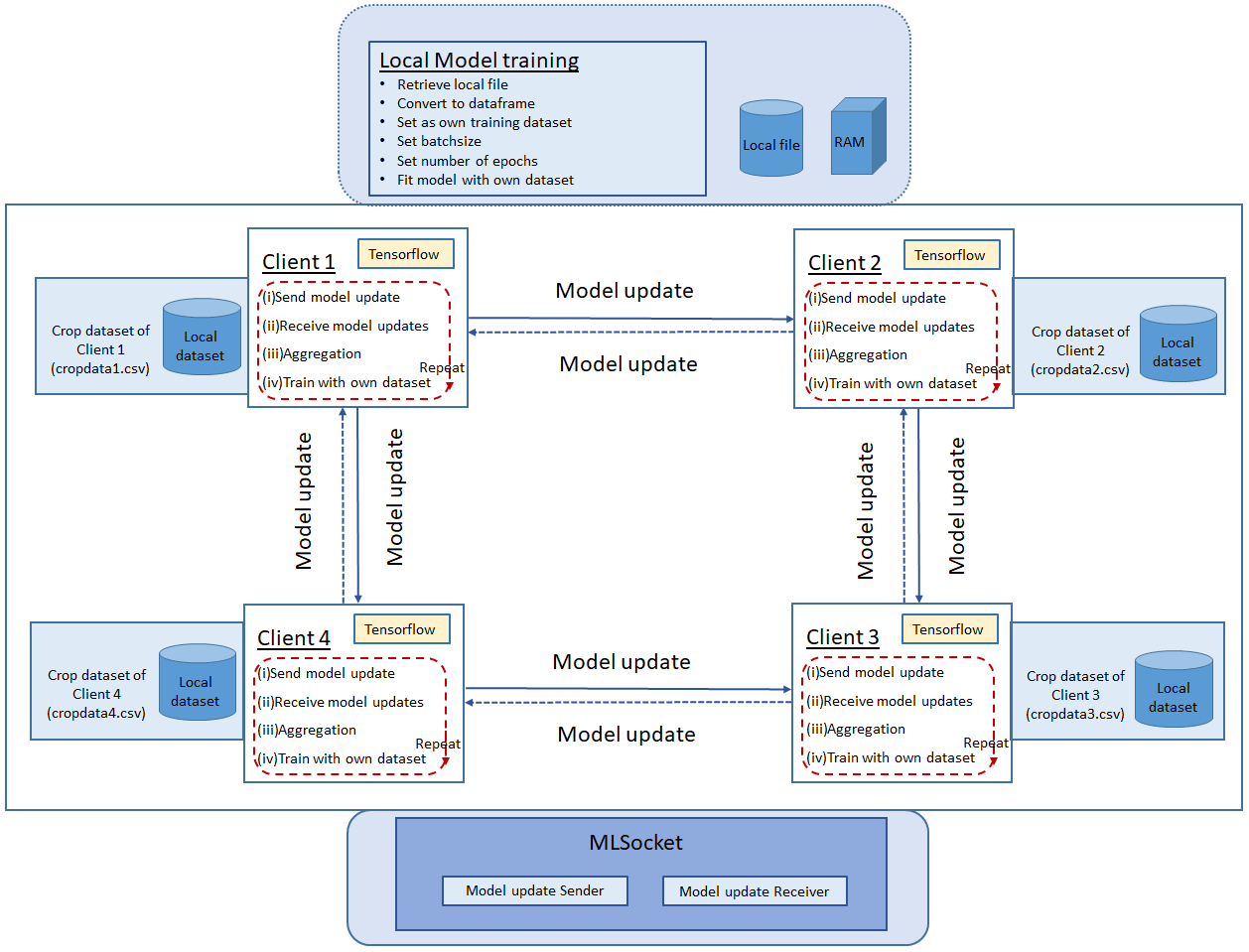}
    \caption{\centering{The experimental implementation diagram of DFL-based crop yield prediction using ring topology}}
    \label{ringimp}
    \end{minipage}
\begin{minipage} {0.49\linewidth}
    \includegraphics[width=0.99\linewidth, height=3.2in]{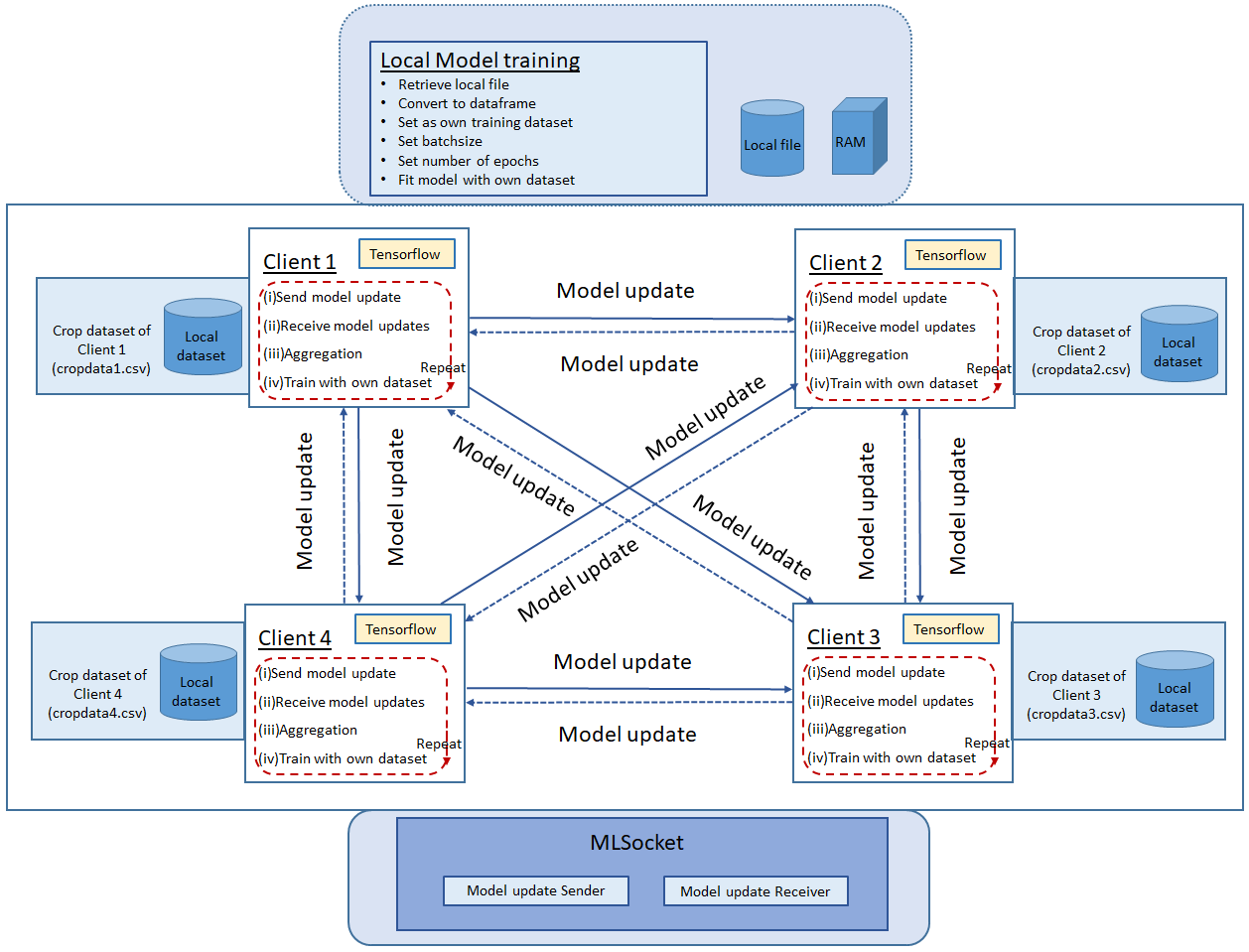}
    \caption{\centering{The experimental implementation diagram of DFL-based crop yield prediction using mesh topology}}
    \label{meshimp}
\end{minipage}
    \centering
\end{figure*}

\subsection{CFL in crop yield prediction}
In CFL, we performed three case studies: (i) Scenario 1: Five instances as client machines and one instance as the server machine, (ii) Scenario 2: Ten instances as client machines and one instance as the server machine, and (iii) Scenario 3: Fifteen instances as client machines and one instance as the server machine. Each of the client machines had its local dataset, and the server machine had the global dataset. The client machines shared their model updates with the server machine. The server machine performed aggregation, and updated the global model. The global model update was then shared with the clients. The maximum number of rounds was considered 10. However, we achieved a global model with accuracy of above 0.95 after two rounds, and the difference between the accuracy level for two consecutive rounds was $<0.001$ after three rounds.
\par
In \textit{Scenario 1}, five instances (H1, H2, H4, H5, and H6) worked as the client machines and shared their model updates with instance H3 acting as the server machine. The server updated its global model after aggregation, and shared the model update with the five clients. The prediction accuracy of the global model before and after FL are presented in Figs. \ref{accg1} and \ref{accg2} respectively. 
\begin{figure*}
    \centering
\begin{minipage}{0.495\linewidth}
\includegraphics[width=0.99\linewidth, height=2.7in]{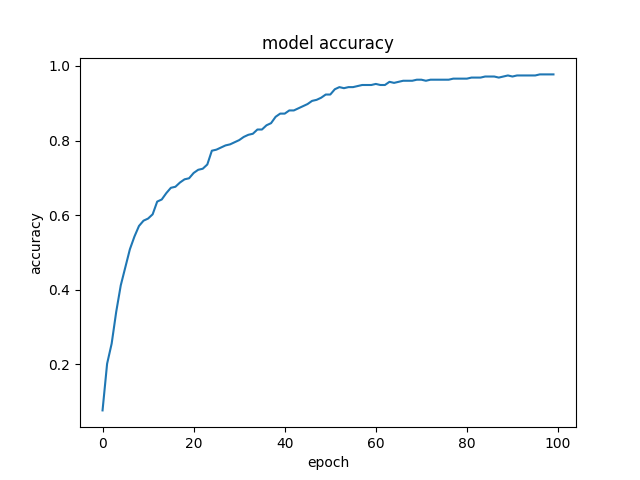}
\caption{\centering{Accuracy of the global model before FL}}
\label{accg1}
\end{minipage}
\begin{minipage}{0.495\linewidth}
\includegraphics[width=0.99\linewidth, height=2.7in]{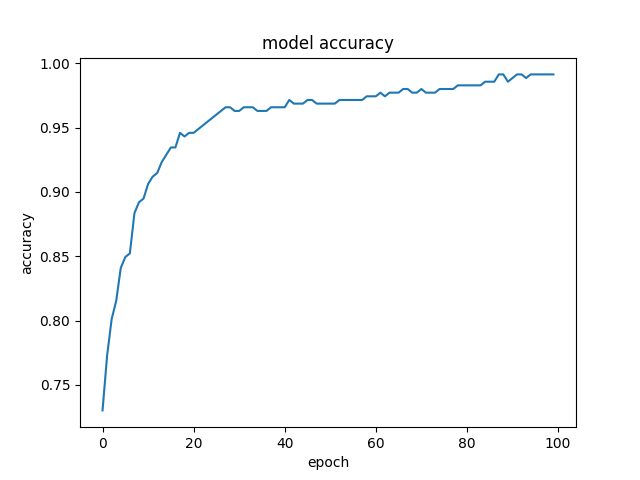}
\caption{\centering{Accuracy of the global model after FL}}
\label{accg2}
\end{minipage}
\end{figure*}
\begin{figure*}
    \centering
\begin{minipage}{0.49\linewidth}
\includegraphics[width=0.99\linewidth, height=2.7in]{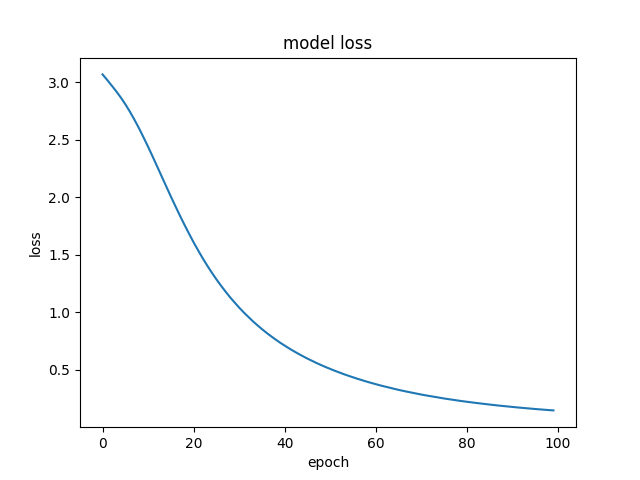}
\caption{\centering{Global loss before FL}}
\label{loss_nofl}
\end{minipage}
\begin{minipage}{0.495\linewidth}
\includegraphics[width=0.99\linewidth, height=2.7in]{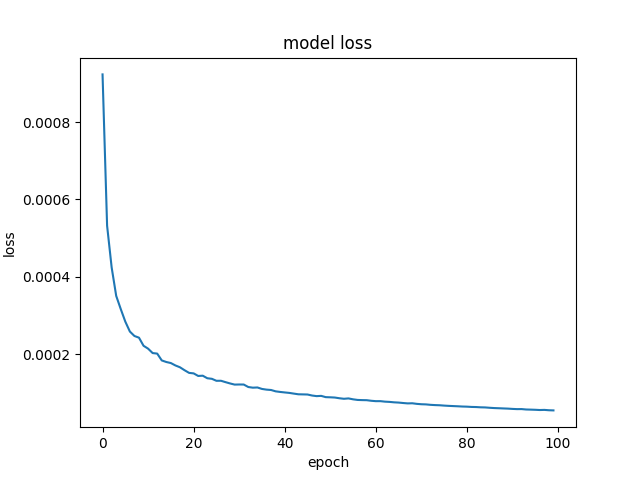}
\caption{\centering{Global loss after using CFL}}
\label{loss_cfl}
\end{minipage}
\end{figure*}
As we observed from the experiment, the global model's accuracy was improved from $\sim$0.93 to $\sim$0.98 using CFL. The global loss with respect to the number of epochs are presented in Figs. \ref{loss_nofl} and \ref{loss_cfl}, without using FL and using CFL, respectively. Without using FL, the loss after 100 epochs is $<$0.5, and using CFL, the global loss is $<$0.0002 after 100 epochs, as we observe from the figures. In Section \ref{convergence}, we have theoretically proved that the global loss tends to 0 as the number of rounds increases. Here, we have presented the result after three rounds, and we observed from the experiment that after 10 rounds the loss is below 0.0001, which is almost negligible.   
\par
We conducted the experiment for \textit{Scenarios 2 and 3} also. The prediction accuracy we obtained for the global model after CFL in both the scenarios is 0.97. The prediction accuracy, precision, recall, and F1-Score for the three scenarios 1, 2, and 3, are presented in Fig. \ref{acc_global}. As we observe, the accuracy, precision, recall, and F1-Score for scenarios 2 and 3 are 0.97, and for scenario 1 all the accuracy metrics provide the value of 0.98. The prediction accuracy, precision, recall, and F1-Score for the clients are presented in Fig. \ref{acc_local}. From the experiment we observed that the accuracy, precision, recall, and F1-Score of all the local models after CFL were $\geq$0.95.  For the local models, the average accuracy obtained for scenarios 1, 2, and 3 were 0.97, 0.96, and 0.96 respectively. We observe from Figs. \ref{acc_global} and \ref{acc_local} that the accuracy, precision, and recall of the local models and the global model in all three scenarios are above 0.95. Further, the F1-Score of the local models as well as the global model are $\geq$0.95 in all three scenarios. 
\par
Each of the clients has its own local dataset. Thus, the weights of different local models are different. The server updates the global model after receiving model updates from all the clients. As the model weights of different clients vary, the accuracy of the global model also can differ for different number of clients participating in CFL. However, as we observed from the experiment for all the three scenarios, the accuracy of the global model has been improved than the global model before using CFL. The local models are also updated after receiving the global model update from the server. As the number of clients participating in the CFL varies, the average accuracy may differ for different number of clients. However, we observe that the accuracy of the local models for all three scenarios are above 0.95. Hence, we can recommend the use of FL to achieve a global model with high prediction accuracy and privacy protection without sharing the dataset. The higher accuracy, precision, recall, and F1-Score of global and local models indicate the models provide accurate prediction with stability. 

\begin{figure*}
    \centering
    \begin{minipage}{0.495\linewidth}
    \includegraphics[width=0.99\linewidth, height=2.5in]{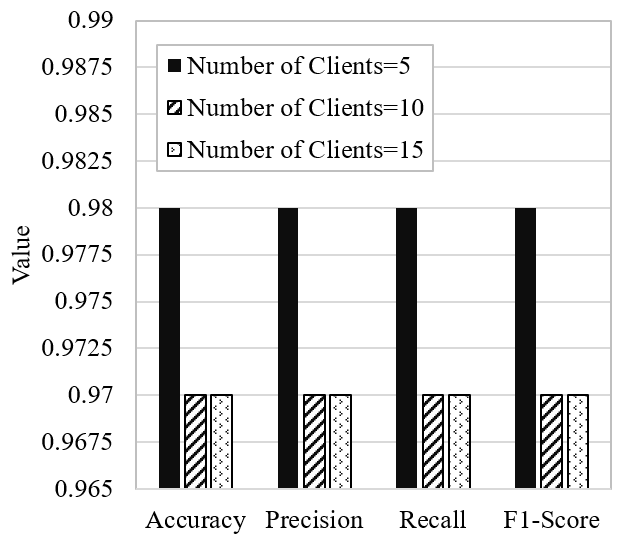}
    \caption{\centering{Accuracy, precision, recall, and F1-Score of the global model after FL}}
    \label{acc_global}
    \end{minipage}
    \begin{minipage}{0.495\linewidth}
    \includegraphics[width=0.99\linewidth, height=2.5in]{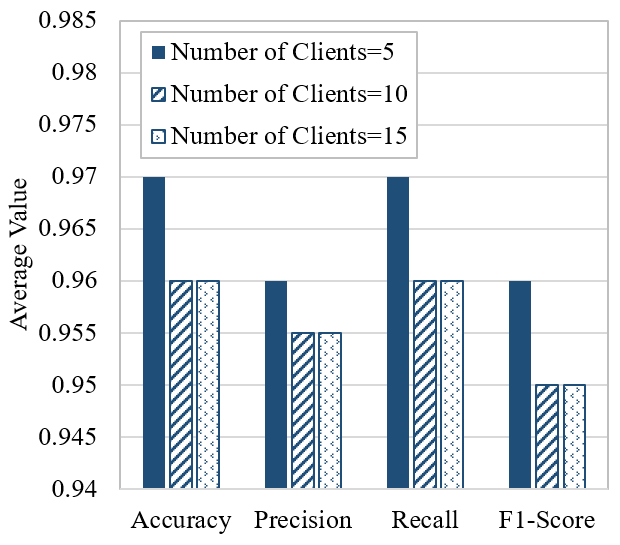}
    \caption{\centering{Average accuracy, precision, recall, and F1-Score of the local models after FL}}
    \label{acc_local}
    \end{minipage}
\end{figure*}
\begin{figure}
    \centering
    \includegraphics[width=0.6\linewidth, height=2.5in]{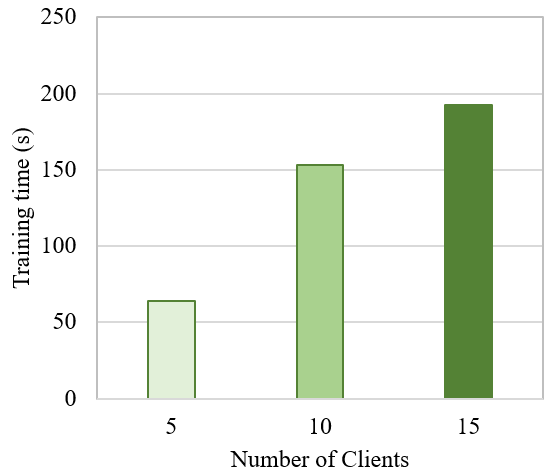}
    \caption{Training time of the global model}
    \label{trtime_gl}
\end{figure}
\begin{figure*}
    \centering
    \begin{minipage}{0.495\linewidth}
    \includegraphics[width=0.99\linewidth, height=2.5in]{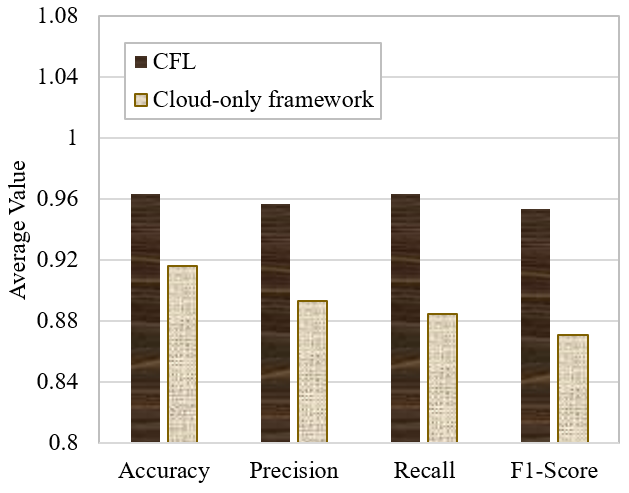}
    \caption{\centering{Comparison of average accuracy, precision, recall, and F1-Score between the CFL-based and cloud-only frameworks}}
    \label{acc_comp}
    \end{minipage}
    \begin{minipage}{0.495\linewidth}
    \includegraphics[width=0.99\linewidth, height=2.5in]{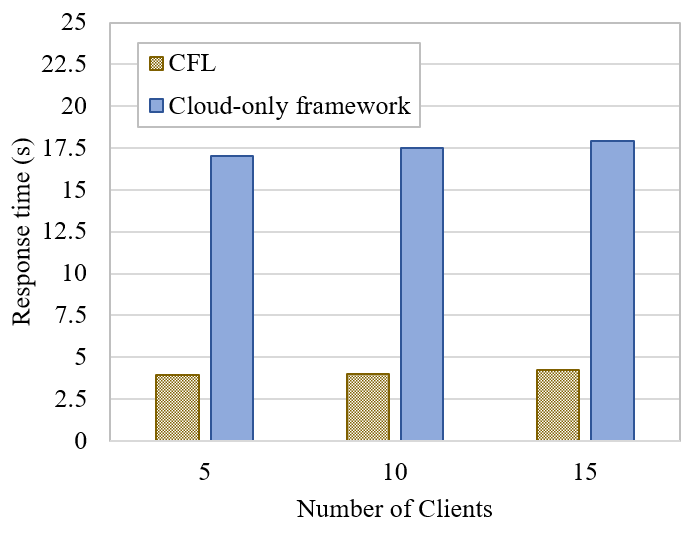}
    \caption{\centering{Comparison of response time between the CFL-based and cloud-only frameworks}}
    \label{resp}
    \end{minipage}
\end{figure*}
The training time for the global model for scenarios 1, 2, and 3 are presented in Fig. \ref{trtime_gl}. The training time is measured in seconds (s). As observed from the results, the training time of the global model for scenarios 1, 2, and 3 are 64.13 s, 153.04 s, and 192.33 s respectively. Here, the training time is measured as the sum of the time consumption in model initialization, local model training, exchanging model updates with participating nodes, and aggregation. As the number of clients increase, the delay in receiving updates from the clients increase though parallel communication takes place. As the aggregation takes place only after receiving updates from all the participating clients, the training time increases with the number of participating clients. Thus, the training time for scenario 1 with five clients is low compared to the other two scenarios. 
\subsubsection{Comparison with Cloud-only framework}
The CFL-based framework is compared with the cloud-only framework with respect to accuracy, precision, recall, F1-Score, and response time. By response time, we refer to the difference between the time stamps of submission of request and receiving response from the device regarding crop yield prediction. In this case, again we conducted the experiment for five, ten, and fifteen clients. Each of the clients sent its dataset to the server for analysis. The time consumption in sending the dataset, performing analysis on the dataset using LSTM, sending result to the client, and receiving result by the client, were measured. We observed that the response time for the cloud-only framework was higher than the CFL-based framework. We also observed that the prediction accuracy for the clients' datasets were also lower than the CFL-based framework. In Fig. \ref{acc_comp}, the average accuracy, precision, recall, and F1-Score of the local models after CFL are compared to the cloud-only framework. Fig. \ref{resp} presents the comparison of the response time of the CFL-based model to the cloud-only framework. The CFL-based framework reduces the response time $\sim$75\% than the cloud-only framework. The average accuracy, precision, recall, and F1-Score are also improved by $\sim$5\%, $\sim$7\%, $\sim$8\%, and $\sim$9\% using CFL than the cloud-only framework. Hence, we observe that the CFL-based framework outperforms the cloud-only framework. 
\begin{figure*} 
    \centering
\begin{minipage}{0.495\linewidth}
\includegraphics[width=0.99\linewidth, height=2.5in]{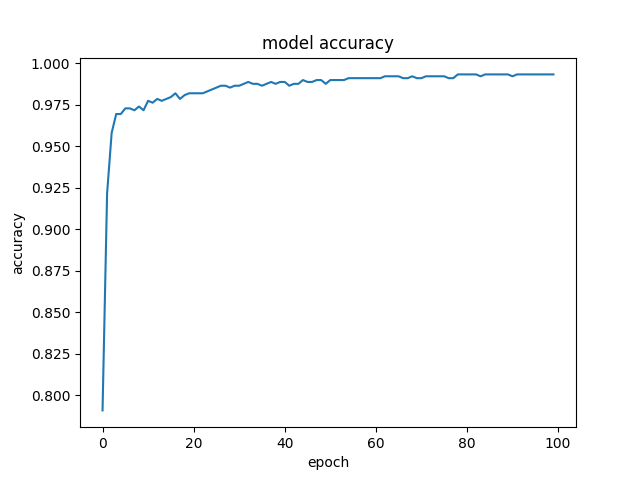}
\caption{\centering{Accuracy of Node1 in ring-based network}}
\label{accg1r}
\end{minipage}
\begin{minipage}{0.495\linewidth}
\includegraphics[width=0.99\linewidth, height=2.5in]{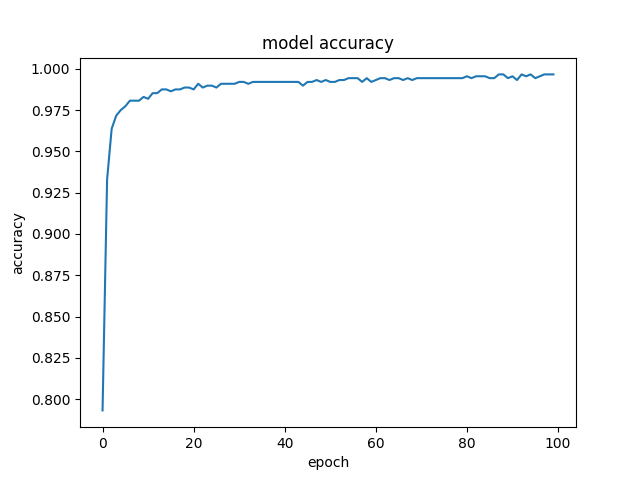}
\caption{\centering{Accuracy of Node2 in ring-based network}}
\label{accg2r}
\end{minipage}
\end{figure*}
\begin{figure*} 
    \centering
\begin{minipage}{0.495\linewidth}
\includegraphics[width=0.99\linewidth, height=2.5in]{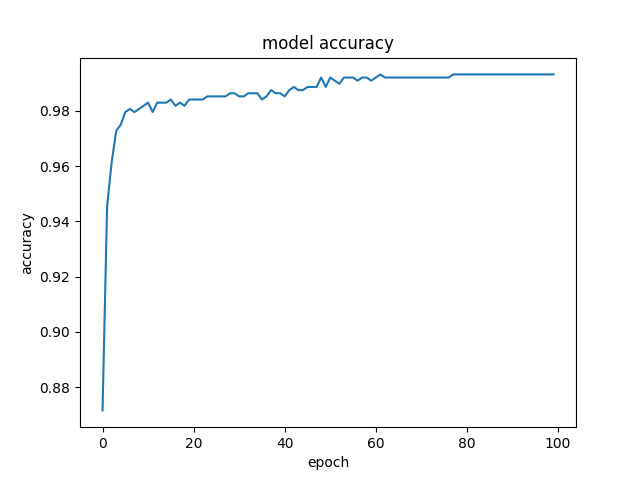}
\caption{\centering{Accuracy of Node3 in ring-based network}}
\label{accg3r}
\end{minipage}
\begin{minipage}{0.495\linewidth}
\includegraphics[width=0.99\linewidth, height=2.5in]{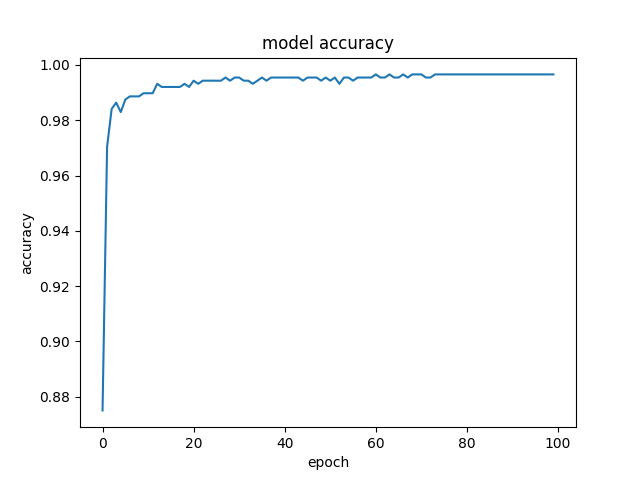}
\caption{\centering{Accuracy of Node4 in ring-based network}}
\label{accg4r}
\end{minipage}
\end{figure*}
\subsection{DFL in crop yield prediction}
The CFL-based framework though has high prediction accuracy, there are several limitations such as increase in overhead on the cloud, security, requirement of high bandwidth, etc. To address these limitations, DFL has been introduced as we have discussed earlier. In DFL, the clients form a collaborative network among themselves. In this work, we used DFL frameworks with ring and mesh topology. Here, also we considered three network scenarios: (i) Four nodes form the network, (ii) Seven nodes form the network, and (iii) Ten nodes form the network. In ring topology, each node exchanges model updates only with the preceding and succeeding nodes. In mesh topology, each node exchanges model updates with rest of the nodes in the formed network.
In \textit{Scenario 1}, we had considered four nodes (H1, H2, H4, and H5), which were connected either using ring topology or mesh topology. H3 worked as the server in our experiment. While using ring topology, the prediction accuracy of all the four nodes after DFL are presented in Figs. \ref{accg1r}, \ref{accg2r}, \ref{accg3r}, and \ref{accg4r}, respectively. The maximum number of rounds we had considered 10, and the number of local epochs were considered 100. 
\begin{figure*}
    \centering
\begin{minipage}{0.495\linewidth}
\includegraphics[width=0.99\linewidth, height=2.5in]{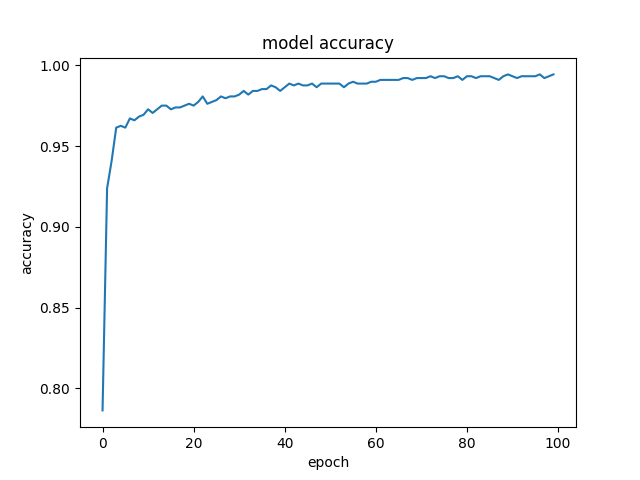}
\caption{\centering{Accuracy of Node1 in mesh-based network}}
\label{accg1m}
\end{minipage}
\begin{minipage}{0.495\linewidth}
\includegraphics[width=0.99\linewidth, height=2.5in]{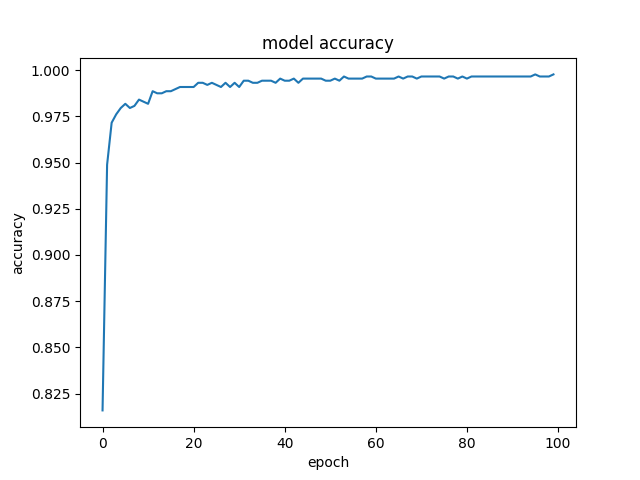}
\caption{\centering{Accuracy of Node2 in mesh-based network}}
\label{accg2m}
\end{minipage}
\end{figure*}
\begin{figure*}
    \centering
\begin{minipage}{0.495\linewidth}
\includegraphics[width=0.99\linewidth, height=2.5in]{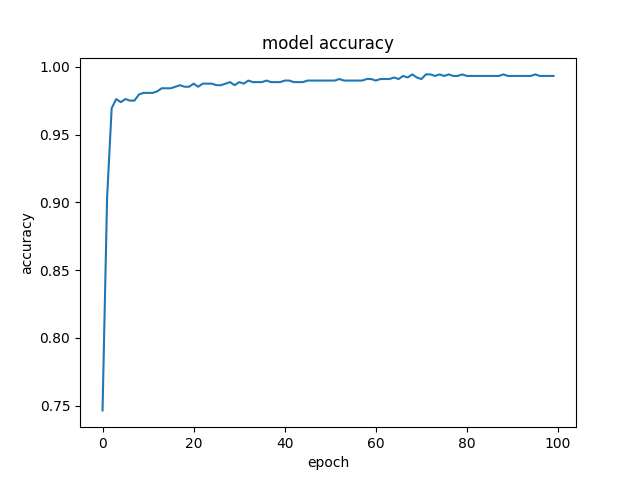}
\caption{\centering{Accuracy of Node3 in mesh-based network}}
\label{accg3m}
\end{minipage}
\begin{minipage}{0.495\linewidth}
\includegraphics[width=0.99\linewidth, height=2.5in]{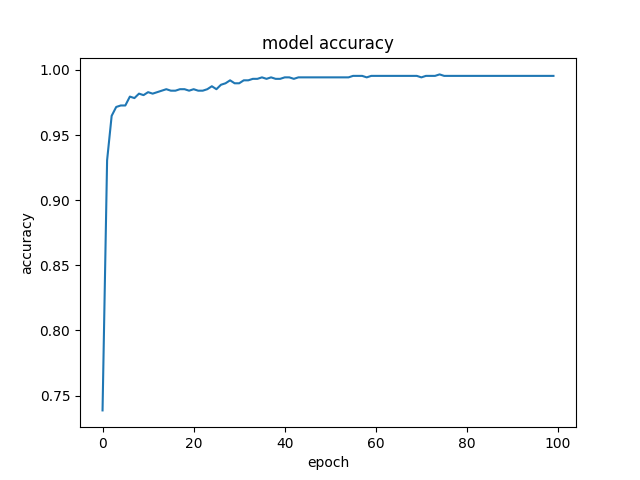}
\caption{\centering{Accuracy of Node4 in mesh-based network}}
\label{accg4m}
\end{minipage}
\end{figure*}
\begin{figure*}
    \centering
    \begin{minipage}{0.495\linewidth}
\includegraphics[width=0.99\linewidth, height=2.4in]{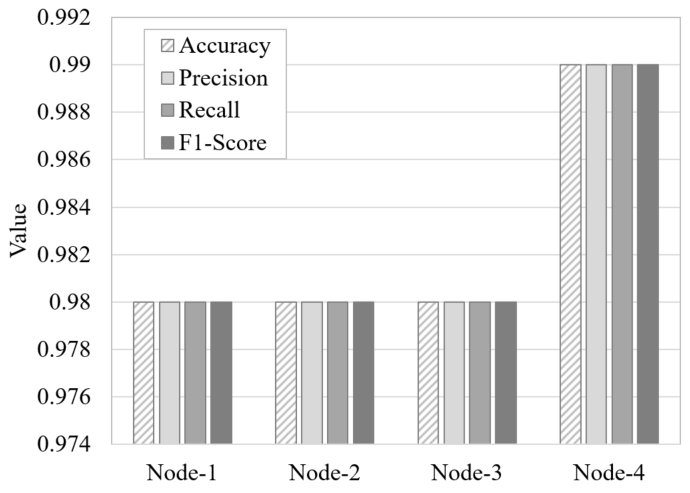}
    \caption{\centering{Accuracy, precision, recall, and F1-Score of the local models in ring-based network}}
    \label{acc_ring}
    \end{minipage}
    \begin{minipage}{0.495\linewidth}
    \includegraphics[width=0.99\linewidth, height=2.4in]{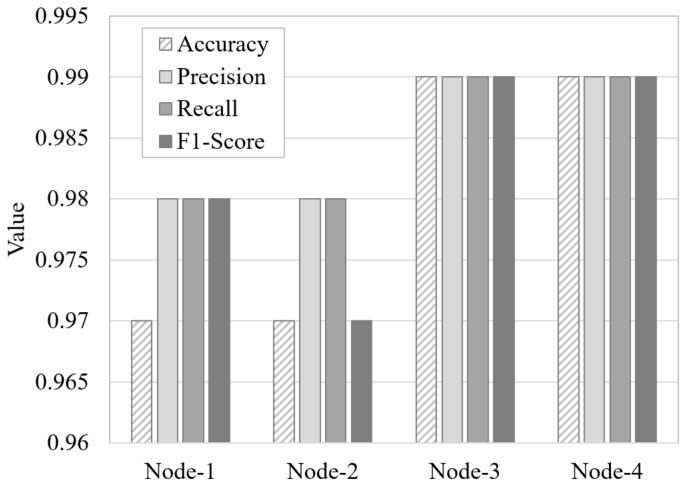}
    \caption{\centering{Accuracy, precision, recall, and F1-Score of the local models in mesh-based network}}
    \label{acc_mesh}
    \end{minipage}
\end{figure*}
In the considered ring-based P2P network, H1 shares its model updates with H2 and H5, H2 shares its model updates with H1 and H4, H4 shares its model updates with H2 and H5, and H5 shares its model updates with H1 and H4. After receiving model updates each of the node performs aggregation and updates their local models accordingly. In the conducted experiment, the prediction accuracy of nodes 1 (H1), 2 (H2), 3 (H4), and 4 (H5) were 0.98, 0.98, 0.98, and 0.99 respectively, while using the ring-based P2P network. As the datasets of the nodes are different, their prediction accuracy may also differ. However, we observed that the prediction accuracy was above 0.97 for all the nodes, and the prediction accuracy of the global model developed by the nodes was 0.98.
\par
The prediction accuracy of all the four nodes while using mesh topology are presented in Figs. \ref{accg1m}, \ref{accg2m}, \ref{accg3m}, and \ref{accg4m}, respectively. 
In the mesh-based network, each node shares its model updates with rest of the nodes. From the results we observe that the prediction accuracy of nodes 1, 2, 3, and 4, after FL were 0.97, 0.97, 0.99, and 0.99 respectively. Here, also the nodes' datasets are different and accuracy level may also therefore vary. However, the accuracy obtained for all the nodes was above 0.96, and the prediction accuracy of the global model developed by the nodes was 0.98. 
\par
The accuracy, precision, recall, and F1-Score for the four nodes in ring topology and in mesh topology are presented in Figs. \ref{acc_ring} and \ref{acc_mesh} respectively. 
We observe from the results that the accuracy, precision, recall, and F1-Score for all the nodes are above 0.96 for both the ring and mesh topology-based networks. Hence, we observe that using DFL high prediction accuracy can be obtained. The training time of the four nodes in ring and mesh-based network are presented in Fig. \ref{trtime_locm}. As we observed in ring-based network the training time of nodes 1, 2, 3 and 4 are 35.53 s, 27.86 s, 31.84 s, and 30.62 s respectively. We also observed that the training time of the nodes 1, 2, 3, and 4 are 36.99 s, 39.65 s, 36.8 s, and 32.97 s, while using the mesh-based network. 
\begin{figure}
\centering
    \includegraphics[width=0.9\linewidth, height=2.5in]{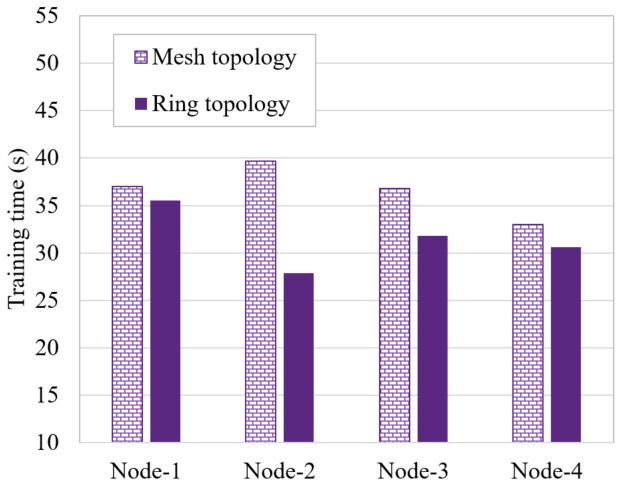}
    \caption{\centering{Training time of the local models in ring and mesh-based networks}}
    \label{trtime_locm}
\end{figure}
\begin{figure*} 
    \centering
\begin{minipage}{0.49\linewidth}
\includegraphics[width=0.99\linewidth, height=2.4in]{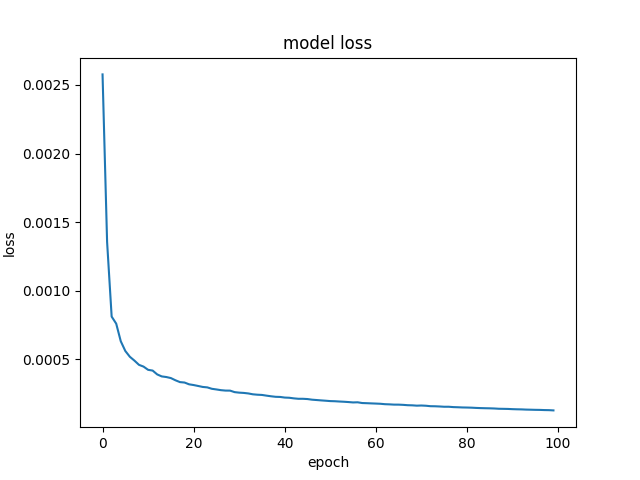}
\caption{\centering{Global loss after using ring-based DFL framework}}
\label{loss_dflring}
\end{minipage}
\begin{minipage}{0.495\linewidth}
\includegraphics[width=0.99\linewidth, height=2.4in]{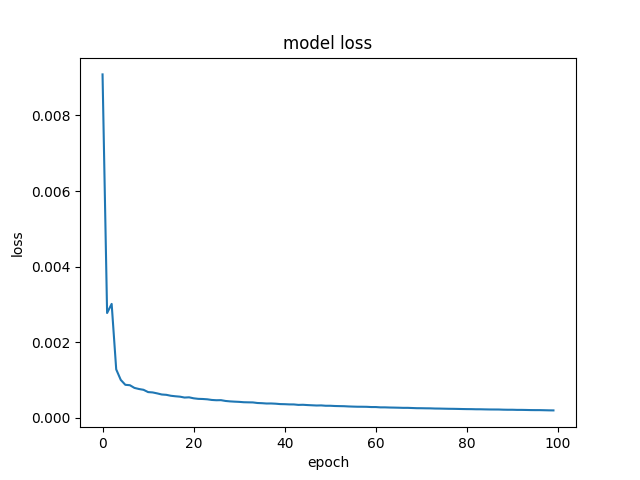}
\caption{\centering{Global loss after using mesh-based DFL framework}}
\label{loss_dflmesh}
\end{minipage}
\end{figure*}
\par
The global loss while using DFL using ring and mesh topology, are presented in Figs. \ref{loss_dflring} and \ref{loss_dflmesh}, respectively. As we observe from the figures, the global loss is $<$0.0005 for ring topology and $<$0.0002 for mesh topology, after 100 epochs. The experimental results after three rounds are presented in the figures. We observed that after 10 rounds the loss is below 0.00002 for both the ring-based and mesh-based DFL frameworks, which is almost negligible. In Section \ref{convergence}, theoretically we have already proved that the global loss tends to 0 as the number of rounds increases. Now, the experimental results also demonstrate the same. 
\par     
We conducted the experiment for seven and ten nodes also. As we observed from the results, the average accuracy, precision, recall, and F1-Score were above 0.98 while using ring topology. We also observed that while using mesh topology the  average accuracy, precision, recall, and F1-Score were above 0.97. However, from the results we observed that for some of the nodes ring topology provided higher accuracy and for other nodes mesh topology provided higher accuracy. However, in both the cases the average accuracy, precision, and recall, and F1-Score were above 0.97. Hence, we observed from the results that for both the ring and mesh topology, high prediction accuracy was obtained by the DFL in all three scenarios. The average accuracy, precision, recall, and F1-Score for the ring-based and mesh-based DFL frameworks are presented in Figs. \ref{acc_ring_avg} and \ref{acc_mesh_avg}, respectively. \textit{The prediction accuracy of the global model developed by the nodes for the three scenarios for mesh topology and ring topology were $\geq$0.98.} The average training time for the local models while using ring and mesh topology are presented in Fig. \ref{trtime_locr}. 
\begin{figure*} 
    \centering
    \begin{minipage}{0.495\linewidth}
    \includegraphics[width=0.99\linewidth, height=2.7in]{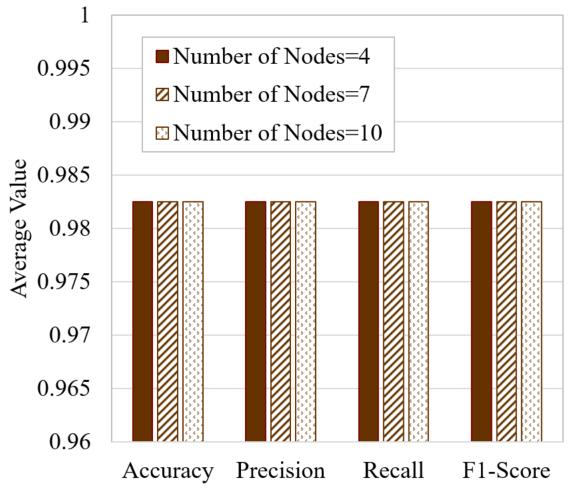}
    \caption{\centering{Average accuracy, precision, recall, and F1-Score of the local models in ring-based network}}
    \label{acc_ring_avg}
    \end{minipage}
    \begin{minipage}{0.495\linewidth}
    \includegraphics[width=0.99\linewidth, height=2.7in]{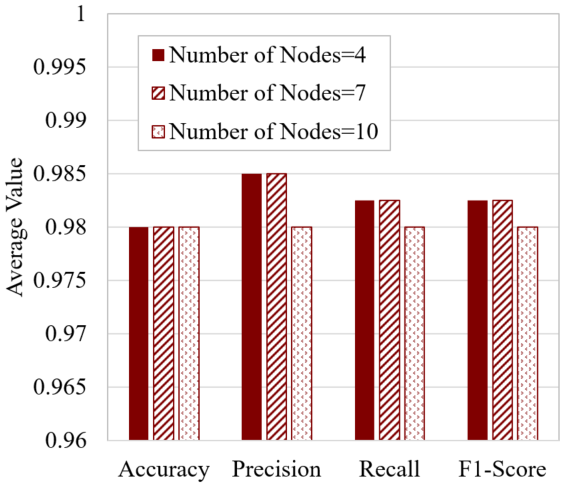}
    \caption{\centering{Average accuracy, precision, recall, and F1-Score of the local models in mesh-based network}}
    \label{acc_mesh_avg}
    \end{minipage}
\end{figure*}

\begin{figure*}
    \centering
    \includegraphics[width=0.5\linewidth, height=2.5in]{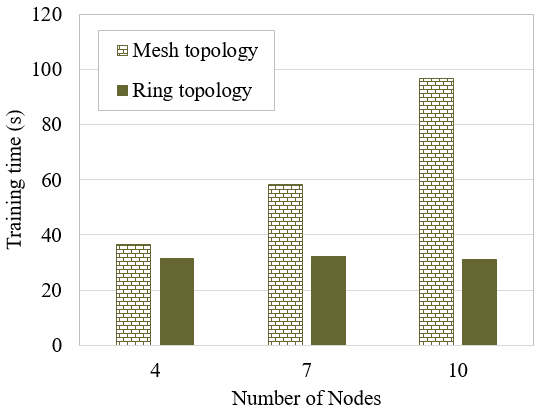}
    \caption{\centering{Average training time of the local models in ring and mesh-based networks}}
    \label{trtime_locr}
\end{figure*}
As observed from the results, the average training time of the local models for scenarios 1, 2, and 3 are 31.46 s, 32.26 s, and 30.98 s, respectively while using ring topology. We also observed that while using mesh topology, the average training time of the local models for scenarios 1, 2, and 3 are 36.6 s, 58.4 s, and 96.71 s, respectively. We observed that for both ring and mesh topology with four, seven, and ten nodes, it took two rounds to get the desired accuracy level of 0.95, and the difference between the accuracy levels attained by two consecutive rounds was $<0.001$ after three rounds for each of the nodes. In ring topology, each node exchanges model updates with the adjacent nodes, whereas in mesh topology, each node exchanges model updates with rest of the nodes in the formed network. As we observed in our experiment, the training time for mesh topology was higher. We also measured the response time for both the mesh-based and ring-based networks, and we observed that for both ring and mesh topology, the response time was in the range of 1.2-3.5s, in all the three cases.
\par
Finally we observe that (i) for CFL we achieved $\geq$97\% accuracy for the global model, and (ii) for DFL using ring and mesh topology we achieved $\geq$98\% and $\geq$97\% prediction accuracy respectively for the local models, and $\geq$98\% accuracy for the global model. As we observe from the experimental results, using CFL and DFL, high prediction accuracy can be obtained without sharing the data. 

\subsection{Comparison with existing crop yield prediction approaches}
In this section, we compare the performance of the CFL and DFL-based frameworks in crop yield prediction with the state-of-the-art models for crop yield prediction. The comparative analysis is presented in Table \ref{tab:comp1}. At first, we compare the CFL and DFL-based frameworks with the existing crop yield prediction frameworks that used the $same\;dataset^1$ that we used for performance analysis. After that we draw a comparison of the CFL and DFL-based frameworks with an existing FL-based framework that used another dataset for performance evaluation.
\begin{sidewaystable}
\caption{\textcolor{black}{Comparison of performance of proposed and existing crop yield prediction frameworks}}
\small
    \centering
    \begin{tabular}{|c|c|c|c|c|c|}
        \hline
        \textbf{Work} &	 \textbf{Classifier} &  \textbf{FL} &	 \textbf{Accuracy} &  \textbf{Training} & \textbf{Response}\\
        & & \textbf{is used}& &  \textbf{time}&  \textbf{time}\\
        \hline
        \cite{thilakarathne2022cloud} & RF (Highest), & No  & 97.18\% & Not & Not \\ 
&  DT, KNN,	 & & &measured& measured \\
&  XGBoost, SVM	 & & &&\\
        \hline
       \cite{bakthavatchalam2022iot} & MLP (Highest), & No & 98.23\%	 & Model build time:& Not  \\
 & Decision Table, JRip & & &10.56 s (MLP)& measured\\
        \hline
       \cite{cruziot} & KNN & No & 92.62\% (Recall)	 & Not & Not \\
   &  & & &measured& measured\\
  \hline
    \cite{kathiria2023smart} & DT, SVM, & No & 99.24\%	 & Not & Not \\
   & KNN, LGBM,& & &measured& measured\\
  & RF (Highest) & & &&\\
   \hline
       \cite{gopi2024red} & LSTM, Bi-LSTM, GRU & No & 98.45\%	 & Not & Not \\
    & & & & measured& measured\\
        \hline 
        \cite{idoje2023federated}& Gaussian NB & Yes & 90\%	 & Not& Not  \\
  & & (CFL)& & measured& measured\\
        \hline 
         FL-based   & LSTM & Yes & CFL: $\geq$97\% (global) & CFL: 60-200 s & CFL: 2.5-5 s\\
 framework &  & (CFL & (5-15 clients), & (5-15 clients), &  (5-15 clients),\\
  &  & and & DFL: $\geq$98\% (local) &DFL: 25-40 s & DFL: 1.2-3.5 s \\
   &  & DFL) & (Ring) & (Ring)& (4-10 nodes)\\
   &  & & (4-10 nodes),& (4-10 nodes),& \\
  &  & & $\geq$97\% (local) &30-100 s &\\
  &  & & (Mesh) &(Mesh)&\\
 &  & & (4-10 nodes),& (4-10 nodes)&\\
 &  & & $\geq$98\% (global)& &\\
 &  & & (Mesh and ring)& &\\
        \hline 
    \end{tabular}
    \label{tab:comp1}
\end{sidewaystable}

\par
In (\cite{thilakarathne2022cloud}), the authors used RF, DT, KNN, XGBoost, and SVM, which are well-known ML models, and among them RF achieved the highest accuracy of 97.18\%. In (\cite{bakthavatchalam2022iot}), MLP was used and 98.23\% accuracy was achieved in crop yield prediction. KNN was adpoted in (\cite{cruziot}) for data analysis and an accuracy of 92.62\% was achieved for crop yield prediction. In (\cite{kathiria2023smart}), the authors used DT, SVM, KNN, LGBM, and RF, in crop yield prediction, and among them RF achieved the highest accuracy of 99.24\%. The authors in (\cite{gopi2024red}) used LSTM, Bi-LSTM, and GRU-based framework for data analysis, and achieved an accuracy of 98.45\% in crop yield prediction. As we observe from Table \ref{tab:comp1}, none of the existing approaches used FL. Most of the existing approaches focused on the use of ML/DL approaches for crop yield prediction without addressing the concern of using cloud-only paradigm for data analysis, such as network connectivity issue, data privacy, response time, etc. To achieve data privacy protection by without sharing data but obtain a model with high prediction accuracy through collaborative training has been addressed in our work. We have explored the use of FL in crop yield prediction through an experimental analysis using multiple clients. As we observe CFL and DFL have achieved $\geq$97\% prediction accuracy but without sharing actual datasets. Hence, we observe that using FL high prediction accuracy can be obtained like the state-of-the-art models but with enhanced data privacy. In (\cite{idoje2023federated}), CFL was used for crop yield prediction based on Gaussian NB, and 90\% prediction accuracy was achieved with Adam optimizer and learning rate 0.001. As we observe, only CFL was used in (\cite{idoje2023federated}), whereas we have used both CFL and DFL in crop yield prediction based on LSTM. Further, we have achieved higher accuracy ($\geq$97\%) than (\cite{idoje2023federated}) (Optimizer: Adam, learning rate: 0.001). Further, the training time for both the CFL and DFL approaches have been determined in our work, and we observe that the training time is medium for the considered scenarios. We also observe that the response time is low for both the CFL and DFL-based strategies. Thus, crop yield prediction with high accuracy but low response time can be achieved using FL-based frameworks. 

\section{Future Research Directions}
\label{future}
In this work, we have concentrated on the use of CFL and DFL in crop yield prediction. However, there still remains several challenges stated as follows. 
\begin{itemize}
    \item Data heterogeneity: Data heterogeneity is a critical issue of FL. As the data is distributed among several clients, it may lead to non-independent and identically distributed and unbalanced datasets. In such a scenario, model training is a challenge and it becomes critical to build a global model with consistent performance across all the clients. 
    \item Use of FTL: The datasets of different clients may have different sample space as well as different feature space. In that case, FTL can be used. In FTL, features from different feature spaces are transferred to the same presentation. Further, for enhancing data privacy and security, gradient updates are encrypted. The use of FTL with gradient encryption in crop yield prediction is a significant research direction.
    \item Resource limitation of user device: The FL encourages local data analysis and collaborative learning. However, the user device may not have sufficient resources for executing an ML/DL algorithm, and the user has to use the cloud server for data analysis. Another difficulty may arise when a device cannot execute its local model due to resource limitation or any other issue. In that case, the model of that device along with the dataset needs to be transferred to a nearby node. In both the scenarios, cryptography or steganography can be used for protecting the data privacy by either encrypting or hiding it inside a media during transmission. 
    \item Enhance security of model parameters and the system: Though, no data is shared and the model updates are exchanged in FL, still there is a possibility of leakage of gradient information. In such a scenario, gradient encryption can be used. Further, blockchain can be integrated with FL for enhancing the security of the entire system.
    \item Communication overhead: In FL, the exchange of model updates during training enhances the communication overhead and latency. Therefore, a trade-off should be maintained between the number of rounds of training the model and communication overhead, so that prediction accuracy can be good but the latency will not be very high.
\end{itemize}

\section{Conclusions}
\label{con}
Crop yield prediction is a crucial area of smart agriculture. In this paper, we have explored the use of CFL and DFL in crop yield prediction based on LSTM. An experimental case study has been conducted, where different number of devices perform collaborative training using CFL and DFL. To implement CFL, a client-server paradigm is developed using MLSocket, and multiple clients are handled by the server. To implement the DFL, a collaborative network is formed using ring topology and mesh topology. In the ring-based P2P network, each node exchanges model updates with the neighbour nodes and performs aggregation to build the upgraded model. In the mesh-based network, each node exchanges model updates with rest of the nodes and performs aggregation to build the upgraded model. The performance of the CFL and DFL-based frameworks are evaluated in terms of prediction accuracy, precision, recall, F1-Score, and training time. The experimental results present that $\geq$97\% prediction accuracy has been achieved using the CFL and DFL-based frameworks. The results also show that the CFL-based framework reduces the response time $\sim$75\% than the cloud-only framework. The average accuracy, precision, recall, and F1-Score are also improved by $\sim$5\%, $\sim$7\%, $\sim$8\%, and $\sim$9\% using CFL than the cloud-only framework. Finally, the future research directions in crop yield prediction are highlighted in this paper.

\section*{Acknowledgements}
This work is partially supported by An ARC Discovery Project (DP240102088).

\bibliographystyle{abbrvnat}
\bibliography{main.bib}

\end{document}